%% file: main_arXiv_paper.tex
\documentclass{article}
\PassOptionsToPackage{numbers, compress}{natbib}
\usepackage[preprint]{neurips_2026}

\usepackage[utf8]{inputenc} 
\usepackage[T1]{fontenc}    
\usepackage{hyperref}       
\usepackage{url}            
\usepackage{booktabs}       
\usepackage{amsfonts}       
\usepackage{nicefrac}       
\usepackage{microtype}      
\usepackage{xcolor}         
\input{math_commands}

\usepackage{algorithm}
\usepackage{algorithmic}

\usepackage{enumitem}
\usepackage{amsthm}
\newtheorem{theorem}{Theorem}[section]
\newtheorem{lemma}[theorem]{Lemma}
\newtheorem{proposition}[theorem]{Proposition}
\newtheorem{corollary}[theorem]{Corollary}
\newtheorem{definition}[theorem]{Definition}
\newtheorem{assumption}[theorem]{Assumption}

\usepackage{xspace}
\makeatletter
\DeclareRobustCommand\onedot{\futurelet\@let@token\@onedot}
\def\@onedot{\ifx\@let@token.\else.\null\fi\xspace}

\def\ie{\emph{i.e}\onedot}

\def\wrt{w.r.t\onedot} 

\makeatother

\title{Theoretical Grounding of Out-Of-Distribution Detection With Reinforcement Learning Optimizer}

\author{
  Salimeh Sekeh\thanks{Equal contribution.} \\
  San Diego State University \\
  \texttt{ssekeh@sdsu.edu}
  \And
  Xin Zhang\footnotemark[1] \\
  San Diego State University \\
  \texttt{xzhang19@sdsu.edu}
}

\begin{document}

\maketitle

\input{0_abstract}

\input{1_introduction}

\input{2_related}

\input{3_method}

\input{4_theory}

\input{5_conclusion}








\input{output.bbl}
\newpage
\input{6_appendix}




\end{document}

%% file: math_commands.tex
\usepackage{amsmath,amsfonts,bm}

\def\eqref#1{eq.~(\ref{#1})}
\def\Eqref#1{Eq.~(\ref{#1})}

\def\1{\bm{1}}

\DeclareMathAlphabet{\mathsfit}{\encodingdefault}{\sfdefault}{m}{sl}
\SetMathAlphabet{\mathsfit}{bold}{\encodingdefault}{\sfdefault}{bx}{n}



%% file: 0_abstract.tex
\begin{abstract}
 Out-of-distribution (OOD) detection in dynamic open-world environments requires a model to continually adapt to evolving data distributions while generalizing to covariate-shifted inputs and rejecting semantic-shifted OOD examples. Most existing OOD detection methods optimize only the current-step objective and do not explicitly account for how post-deployment environment changes affect future OOD behavior. In this paper, we establish a  theoretical grounding for dynamic OOD detection using a reinforcement learning (RL)–guided optimizer that explicitly favors updates that reduce the semantic OOD false positive rate over time. We develop a novel augmented optimizer that uses an RL-guided correction term on top of standard gradient descent (GD) and show its improvement over both future-domain generalization and semantic-OOD rejection. We analyze temporal error decomposition in terms of model-change and environment-change generalization errors and develop a new theoretical framework for comparing the generalization errors under both GD and RL-guided optimizers. 
\end{abstract}

%% file: 1_introduction.tex
\section{Introduction}

Consider an object-recognition system deployed on a mobile robot in a warehouse. 
During commissioning, the system is calibrated on summer imagery;
six months later, lighting conditions, seasonal clutter, and equipment upgrades have quietly shifted the data distribution. 
A model update that minimized last week's classification loss may have simultaneously pushed the detector's energy surface in a direction that weakens its response to genuinely anomalous inputs. 
By the time an out-of-distribution (OOD) object reaches the sensor, the detector has already degraded without any single update being obviously wrong. This failure mode highlights a fundamental challenge in dynamic OOD detection: an update that improves \emph{current-step} performance may inadvertently degrade \emph{future-domain} OOD robustness. 
Most existing OOD detection methods are developed under a static deployment assumption, where the test distribution remains fixed after training~\citep{hendrycks2016baseline, liang2018enhancing, liu2020energy}. 
Even methods that adapt at test time, such as entropy minimization~\citep{wang2020tent} and related test-time adaptation approaches~\citep{niu2022efficient}, optimize \emph{instantaneous}
objectives and do not explicitly account for how parameter updates affect the detector's behavior in subsequent environments. 
As a result, adaptation steps may gradually erode the separation between in-distribution (ID) and semantic-OOD samples, potentially leading to accumulated performance degradation, which is a phenomenon not previously formalized.
{This paper addresses the question:} \emph{can we design
an optimization framework for dynamic OOD detection that explicitly
accounts for the future-domain consequences of each update?}

A natural approach is to incorporate information about future performance into the current update. 
We formalize this idea through a value-function perspective: 
by training a value function that estimates cumulative future OOD performance via temporal-difference (TD) learning, 
we obtain a differentiable proxy for future-domain robustness that can be directly integrated into the optimization process.

Building on this insight, we introduce a \emph{reinforcement-learning (RL)-guided optimization framework} that augments standard gradient descent (GD) with a learned correction term. 
In our formulation, RL is used in a \emph{value-learning sense}: a value function, trained via TD learning, captures cumulative future OOD performance and provides a \emph{differentiable update-shaping signal} through its gradient. 
The value function operates purely through its gradient: at each update step, its gradient is added as a correction to the standard parameter update,
requiring no discrete action selection, no environment model, and no
additional inference-time overhead. 
Unlike general-purpose learned optimizers~\citep{andrychowicz2016learning, metz2019understanding}, the value function is explicitly tied to semantic-OOD false positive rate (FPR), enabling a direct theoretical connection between optimization dynamics and OOD detection performance.

\textbf{Theoretical contributions.} We provide a theoretical framework that explains how and when the RL-guided update improves dynamic OOD detection. 
The analysis proceeds in five steps.
    \textit{Temporal error decomposition} (Def.~\ref{def:temporal_decomp}). We decompose the change in future-domain generalization error into two parts: one caused by the environment shift, which is uncontrollable, 
    and one caused by the model update, which is determined by the optimizer. 
    This decomposition isolates the component that optimization can influence.
    \textit{Gradient alignment} (Lemma~\ref{lem:alignment}). We show that, under a structural condition, the update direction induced by the value function is aligned with directions that reduce future-domain generalization error.
    \textit{Model-change improvement} (Lemma~\ref{lem:model_change}). This alignment implies that the RL-guided update consistently reduces the model-induced component of temporal error compared to standard gradient descent.
    \textit{Energy bridge} (Lemma~\ref{lem:energy_bridge}). We further show that this improvement in parameter updates translates into stronger separation between ID and semantic-OOD samples, leading to lower FPR. 
    \textit{Unified main result} (Thm~\ref{thm:main}). Combining these results, we establish that the RL-guided optimizer improves both future-domain generalization and OOD detection performance relative to gradient descent, up to higher-order effects from environment drift, with the advantage growing across adaptation steps. 

Our analysis relies on structural assumptions capturing the geometric relationship between covariate-domain learning and OOD separation. 
We show that these assumptions admit interpretable instantiations and can be analyzed concretely in simplified model classes. 
In particular, we instantiate the framework on a head-only one-layer transformer~\citep{akyurek2022learning}, where the gradient-conflict condition reduces to an explicit condition on label alignment and feature-space correlation (Prop.~\ref{prop:transformer_conflict}).

Further, we provide theoretical understanding of RL-GD training regimes and analyze the behavior of the RL optimizer during training by \textit{RL-GD loss gap} (Thm~\ref{thm:loss-gap bound}) along with alternative bounds based on the gradient of energy (Cor.~\ref{corollary:0}) and training step-based gradient energy state (Cor.~\ref{corollary:1}). Finally, we show that RL-guided optimizer improves the effective gradient flow (Thm~\ref{training-EGF-Gap} and Thm~\ref{env-EGF-Gap}). 

\textit{Our contributions are listed below:}
\begin{itemize}[nosep,leftmargin=*]
    \item A temporal error decomposition that isolates the optimizer's controllable influence on future-domain OOD robustness.
    \item An RL-guided optimization framework in which a TD-trained value function provides a differentiable, temporally-aware correction signal.
    \item A theoretical guarantee that the RL-guided optimizer achieves lower future-domain generalization error \emph{and} lower semantic-OOD false positive rate than standard gradient descent, with improvements accumulating across environments.
    \item A concrete instantiation on a head-only one-layer transformer, where the gradient-conflict assumption reduces to a verifiable condition on label alignment and feature-space correlation.
    \item A theoretical analysis of the training regimes for our RL-guided optimizer, including a bound on the loss gap between RL and GD optimization, and a demonstration on how the RL-based approach improves effective gradient flow under environment changes and across training updates.
\end{itemize}

%% file: 2_related.tex
\section{Related Work}
Many OOD detection methods assume \emph{static} distributions, including post-hoc scoring approaches such as maximum softmax probability~\citep{hendrycks2016baseline}, input perturbation~\citep{liang2018enhancing}, Mahalanobis distance~\citep{lee2018simple}, and energy-based scores~\citep{liu2020energy}, as well as training-time strategies like outlier exposure~\citep{hendrycks2018deep}. 
Test-time adaptation methods~\cite{wang2020tent,niu2022efficient} address distribution shift by updating models on unlabeled data~\citep{wang2020tent, niu2022efficient}, but optimize \emph{instantaneous objectives} without considering long-term OOD behavior. 
Recent work leverages wild or mixed ID/OOD data to improve detection and generalization~\citep{hendrycks2018deep, bai2023feed}, yet typically operates in batch settings with independent updates. 
Related lines on continual learning and online adaptation study stability under sequential shifts~\citep{kirkpatrick2017overcoming, lopez2017gradient, farajtabar2020orthogonal}, while RL and meta-learning approaches learn adaptive update rules~\citep{andrychowicz2016learning, finn2017model}, and energy-based theories provide geometric interpretations of OOD detection~\citep{liu2020energy}. 
In contrast, our work provides a unified \emph{temporal perspective}, explicitly modeling how optimization dynamics influence future-domain OOD robustness. Extended discussions are in Appendix~\ref{app:related_work}.

%% file: 3_method.tex
\section{Method}\label{sec:method}

We study dynamic OOD detection in a non-stationary environment, where the detector must continually adapt to a changing mixture of ID, covariate-shifted, and semantic-OOD samples. 
Our method augments gradient descent with a learned correction term derived from a value function trained via TD learning. 
Unlike standard RL formulations that rely on explicit action selection, our framework uses RL in a \emph{value-learning sense}: the learned value function does not produce actions or define a control policy; instead, its gradient serves as a differentiable signal that shapes the parameter update. 
This allows the optimizer to incorporate information about cumulative future OOD performance without introducing discrete decisions or environment models.

\noindent{\textbf{Preliminaries: Dynamic OOD detection setup on wild data}.} At each outer step $t$, the detector encounters data drawn from a non-stationary distribution
\begin{equation}
\mathbb{P}_t^{\mathrm{wild}} = (1-\alpha_t^{\mathrm{cov}}-\alpha_t^{\mathrm{sem}})\mathbb{P}^{\mathrm{ID}}
+ \alpha_t^{\mathrm{cov}}\mathbb{P}_t^{\mathrm{Cov}}
+ \alpha_t^{\mathrm{sem}}\mathbb{P}_t^{\mathrm{Sem}},
\label{eq:wild_method}
\end{equation}
where $\mathbb{P}^{\mathrm{ID}}$ denotes the in-distribution data, $\mathbb{P}_t^{\mathrm{Cov}}$ denotes covariate-shifted OOD data that remain in the known label space, and $\mathbb{P}_t^{\mathrm{Sem}}$ denotes semantic OOD data that should be rejected. 
The detector must simultaneously: 
\textit{i)} maintain classification performance on ID and covariate-shifted samples, and 
\textit{ii)} reduce FPR on semantic OOD samples.

Let $f_\theta:\mathcal{X}\to\mathbb{R}^N$ be a classifier and let
$E_\theta(x) = -\log\sum_{y=1}^N \exp\!\big(f_\theta^{(y)}(x)\big)$
be its energy score as the OOD detection statistics.
We train the detector with an OOD-aware objective
\begin{equation}
\mathcal{L}_{OW}(\theta) = \mathcal{L}_{ID}(\theta) + \lambda_{\mathrm{Sem}}\mathcal{L}_{OOD}(\theta),
\label{eq:low_method}
\end{equation}
where $\mathcal{L}_{ID}$ is the cross-entropy on ID data and $\mathcal{L}_{OOD}$ encourages high uncertainty on semantic-OOD samples. 
Optimizing $\mathcal{L}_{\mathrm{OW}}$ alone is insufficient in a dynamic setting:
it ignores how the \emph{update direction} affects future-domain OOD performance.
In particular, updates that minimize the instantaneous objective may reduce future-domain energy separation, leading to gradual degradation in OOD detection performance. 
This observation motivates our approach: we introduce a learned correction term that captures \emph{long-term} OOD behavior and explicitly biases the optimization trajectory toward improving future-domain robustness.

\noindent{\textbf{RL-Guided Optimizer}.} We now describe how the proposed approach is realized as an optimization rule that augments standard gradient descent with a learned, temporally-aware correction.

\noindent{\textit{\underline{Pseudo-OOD partition.}}}
Since semantic-OOD labels are unavailable, we construct a pseudo semantic-OOD subset by thresholding the energy score using a percentile-based threshold calibrated from ID samples. This partition serves only as a training proxy and does not utilize ground-truth OOD labels.\\
\noindent{\textit{\underline{Soft FPR surrogate.}}}
Since the hard false-positive indicator is non-differentiable, we use a smooth surrogate
\begin{equation}
\widehat{\mathrm{FPR}}_{t,k}^{\,\mathrm{soft}}
=
\frac{1}{|B_{t,k}^{\mathrm{sem}}|}
\sum_{x\in B_{t,k}^{\mathrm{sem}}}
\sigma\!\left(\frac{\tau_t - E_{\theta_{t,k}}(x)}{\beta}\right),
\end{equation}
where $\sigma(\cdot)$ is the logistic sigmoid, $\beta>0$ controls the sharpness, and $\tau_t$ is a threshold calibrated from the current ID batch. 
This provides a differentiable approximation of FPR that is used both in the reward and in computing gradients through the state.
\\
\noindent{\textit{\underline{State representation.}}}
We use $t$ to index environment changes (outer steps) and $k$ to index the inner optimization steps used to adapt the model within each environment. 
The detector's behavior is summarized by a low-dimensional state: 
\begin{equation}
s_{t,k} =
\Big[ \mu_E^{\mathrm{id}},\; \mu_E^{\mathrm{sem}},\;
\mathrm{Var}(E^{\mathrm{id}}),\;
\mathrm{Var}(E^{\mathrm{sem}}),\;
\widehat{\mathrm{FPR}}_{t,k}^{\,\mathrm{soft}},\;
\widehat{\mathrm{Acc}}_{t,k}
\Big].
\label{eq:state_method}
\end{equation}
where $\mu_E^{\mathrm{id}}$ and $\mu_E^{\mathrm{sem}}$ denote the mean energy on ID and pseudo semantic-OOD samples, the variance terms capture uncertainty in energy estimation, and $\widehat{\mathrm{FPR}}_{t,k}^{\,\mathrm{soft}}$ and $\widehat{\mathrm{Acc}}_{t,k}$ denote the $t$-detection and classification performance. This low-dimensional state captures the key statistics for OOD separation.

\noindent{\textit{\underline{Reward.}}} We define the reward to directly reflect the OOD detection objective:
$r_{t,k} = -\widehat{\mathrm{FPR}}_{t,k}$.
This choice ensures that the value function is aligned with the theoretical objective. 
In practice, additional stabilization and accuracy terms may be included to improve optimization robustness, but they reduce to the above form when their weights vanish.
\\
\noindent{\textit{\underline{Value function.}}}
We introduce a value function $V_\phi(s_{t,k})$ that estimates cumulative future OOD performance from the current state. 
It is trained via TD learning using transitions $(s_{t,k}, r_{t,k}, s_{t,k+1})$.
RL is used here purely in a \emph{value-learning sense}: the value function does not produce actions; instead, its gradient provides a differentiable signal that shapes the optimization trajectory.
The gradient is computed via chain rule through the differentiable state:
$\nabla_\theta V_\phi(s_{t,k}) = \frac{\partial V_\phi}{\partial s_{t,k}} \cdot \frac{\partial s_{t,k}}{\partial \theta}$.

\noindent{\textit{\underline{Update rule.}}}
At time step $t$ and update step $k$ with learning rate $\eta$, the gradient descent update is 
\begin{equation}
\theta_{t,k+1}^{GD} = \theta_{t,k}^{GD} -
\eta \nabla_\theta \mathcal{L}_{OW}(\theta_{t,k}),
\label{eq:gd_update_method}
\end{equation}
Our RL-guided update augments this with a correction term:
\begin{equation}
\theta_{t,k+1}^{RL} = \theta_{t,k}^{RL}
- \eta \nabla_\theta \mathcal{L}_{OW}(\theta_{t,k})
+ \eta_{RL}\nabla_\theta V_\phi(s_{t,k}).
\label{eq:rl_update_method}
\end{equation}
This can be interpreted as locally optimizing the augmented objective
$\mathcal{L}_{OW_{RL}}(\theta)
= \mathcal{L}_{OW}(\theta) - \lambda_{RL} V_\phi(s(\theta))$.
A key structural condition underlying our analysis is that the gradient of the covariate-domain loss and the gradient of the semantic-OOD energy score are negatively aligned on average. 
Intuitively, this reflects a tension between improving covariate-domain accuracy and maintaining energy separation for semantic-OOD detection. 
We formalize this condition in Assumption~4 and verify it in a concrete model setting.
\\
Our method is a gradient-based optimizer augmented with a learned, temporally aware objective, where the value function's gradient encodes future-domain OOD behavior.



%% file: 4_theory.tex
\section{Theory}
\label{sec:theory}
We analyze why the RL-guided update improves dynamic OOD detection by isolating the optimizer-controllable component of future-domain error and showing that the value gradient systematically reduces it.
 

\noindent\textbf{Setup and Assumptions.}
Let $\theta_{t,0}$ and $\theta_{t,K}$ denote the detector parameters before and after $K$ inner updates at outer step $t$, and write ${\rm GErr}_t(f_\theta)$ for future-domain generalization error. 
For brevity, we denote $\theta_t$ as $\theta_{t,0}$, and $\theta_{t,K}$ as $\theta_{t+1}$.

\begin{definition}[Temporal error decomposition]
\label{def:temporal_decomp}
The one-step future-domain generalization error change satisfies
\begin{equation*}
{\rm GErr}_{t+1}(f_{\theta_{t+1}}) -{\rm GErr}_t(f_{\theta_t})
=
\underbrace{
{\rm GErr}_{t+1}(f_{\theta_{t+1}})-{\rm GErr}_{t+1}(f_{\theta_t})}_{\text{model-change}}
+
\underbrace{{\rm GErr}_{t+1}(f_{\theta_t}) - {\rm GErr}_t(f_{\theta_t})
}_{\text{environment-change}}.
\label{eq:theory_decomposition}
\end{equation*}
\end{definition}

Def.~\ref{def:temporal_decomp} separates the part of temporal error that is induced by external distribution drift from the part induced by the optimizer. 
Because environment-change effects are not controllable by the learning rule, our RL-guided optimizer focuses on reducing the model-change term. While the existing work~\cite{zhang2024best, temp-scone2025} analyzes bounds for environment-change generalization error, we provide a theoretical analysis of the model-change generalization error, highlighting the differences between standard gradient descent and our RL-guided approach. 


We state the assumptions needed for the analysis. 
Full discussion is deferred to Appendix~\ref{app:assumptions}.

\begin{assumption}[Optimization regularity]
\label{ass:smooth} We assume that 
$\mathcal{L}_{OW}$ and ${\rm GErr}_{t+1}$ are $\beta$-smooth and $L_G$-Lipschitz in $\theta$. 
We also assume the step sizes are small with $\eta_{\rm RL} \le \eta$, and $\eta + \eta_{\rm RL} \le \frac{1}{2\beta}$.
\end{assumption}

\begin{assumption}[Energy-FPR coupling]
\label{ass:energy-fpr}
FPR is coupled to the expected energy on semantic-OOD samples through a differentiable decreasing function $\phi$, \ie, 
${\rm FPR}_t(\theta) \le \phi(\mathbb{E}_{x\sim P_t^{\rm sem}}
[E_\theta(x)])$.
\end{assumption}

\begin{assumption}[Reward monotonicity]
\label{ass:reward-mono}
The value function is monotone \wrt FPR:
$\frac{\partial V}{\partial{\rm FPR}_t} \le -c < 0$, and partial derivatives of $V$ are uniformly bounded.
\end{assumption}

\begin{assumption}[Gradient conflict]
\label{ass:grad-conflict}
On average, the gradient of future covariate-domain cross-entropy is negatively aligned with the gradient of semantic-OOD energy. 
That is, for some $\rho_{\min}>0$, uniformly over $x \sim P_t^{\rm sem}$ and $t \le T$:
\begin{equation}
\mathbb{E}_{x' \sim P_{t+1}^{\rm cov}}
\bigl[
\langle \nabla_\theta L_{\rm CE}(f_\theta(x')),\,
\nabla_\theta E_\theta(x) \rangle
\bigr]
\le -\rho_{\min}\|\nabla_\theta E_\theta(x)\|.
    \label{eq:theory_gradient_conflict}
\end{equation}
\end{assumption}

\begin{assumption}[Slow drift]
\label{ass:slow-dist}
The environment changes slowly across outer steps, so that
$d_{\rm TV}(P_{t+1}^{\rm sem}, P_t^{\rm sem}) \le \xi_t
\le C_{\rm dist}\,\eta_{\rm RL}$.
\end{assumption}

\begin{assumption}[Slow error drift]
\label{ass:slow-gerr}
Slow environment-induced error drift: 
$|{\rm GErr}_{t+1}(f_{\theta_t}) - {\rm GErr}_t(f_{\theta_t})|
\le C_{\rm gerr}\,\eta_{\rm RL}$.
\end{assumption}

\begin{assumption}[Bounded gradients]
The gradients of the objective, value function, and energy score are uniformly bounded:
\label{ass:bounded-grad}
$\|\nabla_\theta \mathcal{L}_{OW}\| \le G$,\;
$\|\nabla_\theta V\| \le G_V$,\;
$\|\nabla_\theta E_\theta(x)\| \le G_E$.
\end{assumption}

\begin{assumption}[Energy smoothness and non-degeneracy]
\label{ass:nondeg}
The energy function is smooth and non-degenerate:
$E_\theta(x)$ is $\beta_E$-smooth in $\theta$, the density at the threshold satisfies $p(\tau) > 0$, and
$\mu_{\min} := \min_{t \le T,\,k<K}
\left\|\mathbb{E}_{x\sim P_t^{\rm sem}}
[\nabla_\theta E_{\theta_{t,k}^{\rm GD}}(x)]\right\| > 0$.
\end{assumption}

\noindent\textbf{Remark.}
Assumptions~\ref{ass:slow-dist} and~\ref{ass:slow-gerr} are logically independent: the first is a distributional-geometry condition used in Lemma~\ref{lem:alignment}; 
the second is an arithmetic condition on the loss landscape used in Thm.~\ref{thm:main}. 
We state them separately to avoid conflating the two.

Since the full state $s_t$ contains multiple statistics, but the analysis depends primarily on two key components: the semantic-OOD energy shift and FPR. 
We thus restrict attention to a reduced state representation, 
$z_{t,k} := \mathbb{E}_{x \sim P_{t+1}^{\rm sem}}[E_{\theta_{t,k}}(x)] - \mathbb{E}_{x \sim P_t^{\rm sem}}[E_{\theta_{t,k}}(x)]$
which captures the change in semantic-OOD energy across environments. 

\subsection{RL Value Gradient Aligns with Future-Domain Improvement}
\label{subsec:theory_alignment}
We next explain why the RL correction term provides a useful update direction. 
The key point is that the value function depends on OOD-relevant state statistics, in particular the semantic-OOD energy and FPR. 
Under the assumptions above, this induces a systematic negative alignment between the value gradient and the future-domain generalization gradient.

\begin{lemma}[Gradient alignment]
\label{lem:alignment}
Under Assumptions~\ref{ass:reward-mono},\ref{ass:grad-conflict}, \ref{ass:slow-dist}, and \ref{ass:bounded-grad},
for every $t$ and $k$,
\begin{equation*}
\big\langle \nabla_\theta {\rm GErr}_{t+1}(\theta_{t,k}), \nabla_\theta V(s_{t,k})\big\rangle
\le
-\kappa(\theta_{t,k}) + \mathcal{O}(\xi_t),
\label{eq:theory_alignment}
\end{equation*}
where
$\kappa(\theta_{t,k}) := c\,p(\tau)\,\rho_{\min}\,
\mathbb{E}_{x\sim \mathbb{P}_t^{Sem}}
\big[\|\nabla_\theta E_{\theta_{t,k}}(x)\|\big]
>0$.
\end{lemma}
\begin{proof}[Proof sketch]
We decompose the value gradient through the state into a semantic-drift component and an FPR component. 
The drift component is small, $\|\nabla_\theta z_{t,k}\| = O(\xi_t)$, by Assumption~\ref{ass:slow-dist}. 
For the FPR component, $\nabla_\theta {\rm FPR}_{t,k}
= -p(\tau)\,\mathbb{E}_{P_t^{\rm sem}}[\nabla_\theta E_{\theta_{t,k}}]$.
By Assumption~\ref{ass:grad-conflict} and Fubini's theorem~\cite{}, its inner product with $\nabla_\theta {\rm GErr}_{t+1}$ is strictly positive, and since $\frac{\partial V}{\partial{\rm FPR}_{t,k}} < 0$, this yields a negative term. Combining the two parts gives the result.
\end{proof}
Lemma~\ref{lem:alignment} shows that the RL correction is not an arbitrary perturbation of the update, but a directionally meaningful term that reduces future-domain error up to a small drift-induced remainder.

\subsection{Model-Change Improvement and Energy Bridge}
\label{subsec:theory_model_change_energy_bridge}

We first turn Lemma~\ref{lem:alignment} into a quantitative improvement on the model-change term in Def.~\ref{def:temporal_decomp}. 
For $\square\in\{\rm GD,RL\}$, we denote the cumulative model-change error at outer step $t$ as
$\Delta_t^\square :=
{\rm GErr}_{t+1}(f_{\theta_{t,K}^{\square}}) - {\rm GErr}_{t+1}(f_{\theta_{t,0}^{\square}})$, when both optimizers start with the shared initialization $\theta_{t,0}$.

\begin{lemma}[RL tightens the model-change]
\label{lem:model_change}
Under Assumptions~\ref{ass:smooth}, \ref{ass:reward-mono}, \ref{ass:grad-conflict}, \ref{ass:slow-dist}, and \ref{ass:bounded-grad}, below holds 
\begin{equation*}
\Delta_t^{\rm RL} \le
\Delta_t^{\rm GD} - \eta_{RL}\sum_{k=0}^{K-1}\kappa(\theta_{t,k}) + \mathcal{O}(K^2\eta \,\eta_{\rm RL}).
\label{eq:theory_model_change}
\end{equation*}
\end{lemma}
\begin{proof}[Proof sketch]
Using $\beta$-smoothness of ${\rm GErr}_{t+1}$ at $\theta_{t,K}^{\rm GD}$ yields a linear term in the parameter gap $\theta_{t,K}^{\rm RL}-\theta_{t,K}^{\rm GD}$ and a quadratic remainder.
A standard recursion shows $\|\theta_{t,k}^{\rm RL}-\theta_{t,k}^{\rm GD}\|=\mathcal{O}(k\eta_{\rm RL})$, so
$\theta_{t,K}^{\rm RL}-\theta_{t,K}^{\rm GD} =
\eta_{\rm RL}\sum_{k=0}^{K-1}\nabla_\theta V(s_{t,k})
+ \mathcal{O}(K^2\eta\,\eta_{\rm RL})$.
Applying Lemma~\ref{lem:alignment} to the linear term yields the gain, and all remaining terms are absorbed into $\mathcal{O}(K^2\eta\,\eta_{\rm RL})$.
\end{proof}

Lemma~\ref{lem:model_change} shows that the RL correction improves the model-change component of temporal generalization error within each adaptation window. 
The improvement is governed by the cumulative value-gradient contribution, while the higher-order remainder $\mathcal{O}(K^2\eta\,\eta_{\rm RL})$ depends only on the relative step sizes.
In particular, when $\eta \asymp \eta_{\rm RL}$, this reduces to $\mathcal{O}(K^2\eta_{\rm RL}^2)$.


We next show how this parameter-space advantage translates into improved OOD detection behavior. 
Since the expected semantic-OOD energy is a smooth function of the model parameters, a first-order expansion relates the energy difference between RL and GD to the inner product between the energy gradient and the parameter gap, up to a higher-order remainder.


\begin{lemma}[Energy bridge]
\label{lem:energy_bridge}
Under Assumptions~\ref{ass:smooth}, \ref{ass:reward-mono}, \ref{ass:grad-conflict}, and~\ref{ass:bounded-grad}-\ref{ass:nondeg},
with shared initialization $\theta_{t,0}$ and sufficiently small $\eta_{\rm RL}$,
\begin{equation}
\mathbb{E}_{x\sim P_t^{\rm sem}}[E_{\theta_{t,K}^{\rm RL}}(x)]-
\mathbb{E}_{x\sim P_t^{\rm sem}}[E_{\theta_{t,K}^{\rm GD}}(x)]
\ge
\eta_{\rm RL} K c\,p(\tau)\,\mu_{\min}^2
- \mathcal{O}(K^2\eta\,\eta_{\rm RL})> 0.
\label{eq:theory_energy_bridge}
\end{equation}
By Assumption~\ref{ass:energy-fpr}, this yields ${\rm FPR}_t^{\rm RL} \le {\rm FPR}_t^{\rm GD}$.
\end{lemma}

\begin{proof}[Proof sketch]
Using $\beta_E$-smoothness, the expected energy difference admits a first-order expansion in the parameter gap:
$\mathbb{E}_{P_t^{\rm sem}}[E_{\theta^{\rm RL}}] - \mathbb{E}_{P_t^{\rm sem}}[E_{\theta^{\rm GD}}] =
\big\langle \mathbb{E}_{P_t^{\rm sem}}[\nabla_\theta E],\,
\theta^{\rm RL} - \theta^{\rm GD} \big\rangle
+ \mathcal{O}(\|\theta^{\rm RL} - \theta^{\rm GD}\|^2)$.
Substituting the parameter gap from Lemma~\ref{lem:model_change} reduces this to a sum of inner products between $\nabla_\theta E$ and $\nabla_\theta V$, up to an $\mathcal{O}(K^2\eta\,\eta_{\rm RL})$ remainder.
By the same FPR-channel argument as in Lemma~\ref{lem:alignment}, each term is lower bounded by $c\,p(\tau)\,\mu_{\min}^2 + \mathcal{O}(\eta_{\rm RL})$, where both expectations in the inner product are evaluated at the same $\theta_{t,k}^{\rm GD}$, so the inner product reduces to the squared norm.
The result is yielded after summing over $k$.
\end{proof}
By Assumption~\ref{ass:energy-fpr}, this directly implies a lower semantic-OOD FPR for the RL-guided detector.

\subsection{Main Guarantee Over Time}
\label{subsec:theory_main}

We now combine the temporal decomposition, the one-step model-change improvement, and the energy bridge. 
The remaining issue is that, across outer steps, the RL and GD trajectories no longer start from exactly the same parameters. 
This produces a drift term that accumulates over time. 
We define the minimum gain constant $\kappa_{\min} := c\,p(\tau)\,\rho_{\min}\,\mu_{\min} > 0$,
which characterizes the strength of the RL-induced improvement across updates.

\begin{theorem}[RL improves generalization and OOD rejection]
\label{thm:main}
Under Assumptions~\ref{ass:smooth}-\ref{ass:nondeg}, 
\begin{align}
{\rm GErr}_{t+1}^{\rm RL} &\le {\rm GErr}_{t+1}^{\rm GD}
- \eta_{\rm RL}\kappa_{\min}(t+1)K
+ \mathcal{O}(tK^2\eta\,\eta_{\rm RL})
+ \mathcal{O}(t^2 K L_G \eta_{\rm RL}), \label{eq:theory_main_gerr}\\
{\rm FPR}_t^{\rm RL}
&\le {\rm FPR}_t^{\rm GD}
- c\,p(\tau)\,\mu_{\min}^2\,K\,\eta_{\rm RL}
+ O(K^2\eta\,\eta_{\rm RL}).
\label{eq:theory_main_fpr}
\end{align}
In particular, the generalization gap is strictly negative whenever the cumulative RL gain dominates the drift term. 
The FPR gap is strictly negative whenever the leading energy-improvement term dominates the $\mathcal{O}(K^2\eta\,\eta_{\rm RL})$ remainder.
\end{theorem}

\begin{proof}[Proof sketch]
Telescoping Def.~\ref{def:temporal_decomp} over $\tau=0,\dots,t$ and subtracting the GD equation from RL, the initial terms cancel.
Lemma~\ref{lem:model_change} contributes the gain
$-\eta_{\rm RL}\kappa_{\min}(t+1)K$ plus $\mathcal{O}(tK^2\eta\,\eta_{\rm RL})$.
The environment-change difference is bounded by $\mathcal{O}(t^2 K L_G\eta_{\rm RL})$ via $L_G$-Lipschitz continuity of
${\rm GErr}$ and the starting-point drift $\delta_\tau = \mathcal{O}(\tau K\eta_{\rm RL})$ from iterating Lemma~\ref{lem:model_change}. 
Together these give~\Eqref{eq:theory_main_gerr}.
\Eqref{eq:theory_main_fpr} follows from
Lemma~\ref{lem:energy_bridge} applied at each outer step.
\end{proof}
\noindent\textbf{Remark.} 
The gain grows as $\Theta(tK\eta_{\rm RL})$ while the drift grows as $\mathcal{O}(t^2 K L_G\eta_{\rm RL})$, so the RL advantage is most pronounced in the early-to-mid adaptation regime.
\\
Theorem~\ref{thm:main} provides the main theoretical justification for our method: the RL-guided correction simultaneously improves future-domain generalization and semantic-OOD rejection over time, relative to standard gradient descent.

\subsection{Example: Head-Only One-Layer Transformer}
\label{subsec:theory_transformer_example}

We show a concrete specialization that makes the abstract gradient-conflict assumption directly interpretable. 
Consider a head-only one-layer transformer setting, where
$f_W(x)=Wh(x),
E_W(x) = -\log\sum_{j=1}^N \exp\!\big(w_j^\top h(x)\big)$,
and only the classification head $W$ is updated.

In this setting, the gradients of the energy and cross-entropy loss admit closed-form expressions, 
$\nabla_W E_W(x)=-\,p(x)h(x)^\top,
\nabla_W L_{CE}(W;x,y)=\bigl(p(x)-e_y\bigr)h(x)^\top$,
where $p(x)$ is the softmax output, and $e_y$ denotes the one-hot encoding of label $y$. 
Taking their Frobenius inner product gives 
\begin{equation}
\big\langle \nabla_W L_{CE}(W;x,y),\nabla_W E_W(x')\big\rangle_F =
-\bigl(p(x)-e_y\bigr)^\top p(x')\, h(x)^\top h(x').
\label{eq:transformer_example_conflict}
\end{equation}
where $(x,y)\sim P^{\mathrm{cov}}_{t+1}$ and $x'\sim P^{\mathrm{sem}}_t$ are drawn independently.
This expression reveals that the abstract gradient-conflict condition depends on two factors: \textit{i)} the alignment between predictions and labels, through $(p(x)-e_y)$, and 
\textit{ii)} the similarity between feature representations, through $h(x)^\top h(x')$.

\begin{proposition}
\label{prop:transformer_conflict}
In the head-only one-layer-transformer setting, 
the gradient-conflict assumption holds whenever
$\mathbb{E}_{(x,y)\sim P^{cov}_{t+1}}
\mathbb{E}_{x'\sim P_t^{sem}}
\!\left[(e_y - p(x))^\top p(x')\cdot h(x)^\top h(x')
\right] \ge m > 0$, 
with $\rho_{\min} = \frac{m}{R}$ with 
$R = \sup_{x'}\|p(x')\|\|h(x')\|$.
\end{proposition}

\begin{proof}[Sketch]
From \Eqref{eq:transformer_example_conflict}, taking expectations yields $\mathbb{E}[\langle \nabla_W L_{\rm CE}, \nabla_W E_W\rangle_F] \le -m$.
Since $\|\nabla_W E_W(x')\|_F \le R$, dividing gives
$\rho_{\min} = \frac{m}{R}$.
\end{proof}
\noindent{\bf Discussion.}
This example shows that gradient conflict arises from the interaction between prediction alignment and feature similarity. 
In particular, Prop.~\ref{prop:transformer_conflict} provides a simple, computationally verifiable sufficient condition for
Assumption~\ref{ass:grad-conflict}.

\section{Understanding the Role of Training Regimes}\label{training-regims-sec}
The goal of this section is to analyze the behavior of the RL optimizer in comparison to the GD optimizer across training steps. A key assumption throughout this section is that the learning dynamics can be characterized using a second-order Taylor expansion of the loss function during model training. 
Here we introduce the effect of loss gap between GD vs RL training on a sequence of environments.
\begin{definition}[Loss gap by optimizers] \label{loss-gap}
For the environment $t$, $F_t^{OW}$ evaluates the loss gap of the optimized parameter $\theta^*$ under both GD and RL optimizers  (i.e. $\theta_{t,*}^{\text{RL}}$ and $\theta_{t,*}^{\text{GD}}$):
$$F^{OW}_{t}:= \mathcal{L}_{OW}(\theta_{t,*}^{\text{RL}})-\mathcal{L}_{OW}(\theta^{\text{GD}}_{t,*}).$$ 
\end{definition}

\begin{definition}[Training step-based gradient energy state]\label{def:epoch-based energy gradient state} 
For the environment $t$,, we define the training step-based gradient energy state at optimization step $k$, denoted by $S^{(k)}_{t,\nabla}$, as follows
    \begin{align}
  S^{(k)}_{t,\nabla} :=  \mathbb E_{x\sim {P}^{sem}_{t+1}}\left[\nabla^{(k)}_\theta E_{\theta_t}(x)\right]-\mathbb E_{x\sim {P}^{sem}_{t}}\left[\nabla^{(k)}_\theta E_{\theta_t}(x)\right], 
    \end{align}
where $\nabla^{(k)}_\theta E_{\theta_t}(x)\approx E_{\theta_{t,k}}(x)-E_{\theta_{t,k-1}}(x)$. 
\end{definition}
The assumptions below capture the conditions under which both GD and RL optimizers are trained. Lipschitz continuity of the gradient (Assumption~\ref{ass:Lipschitz}) bounds gradient magnitude and stabilizes optimization, compactness (Assumption~\ref{ass:Compactness}) prevents extreme changes in the parameter values $\theta$ during training, and the slow generalization drift condition (Assumption~\ref{ass:Slow generalization differ}) navigates the influence of semantic OOD data change on the gradient of the energy function.

\begin{assumption}[Lipschitz continuous gradient]
\label{ass:Lipschitz}
We assume that the function $V(s_t)$ (and $\mathcal{L}_{OW}$) is $L_V$-Lipschitz (and $L$-Lipschitz) continuous gradient. 
\end{assumption}

\begin{assumption}[Compactness]
\label{ass:Compactness}
At each training update $k$ both $\theta^{\text{GD}}_{t,k}, \theta^{\text{RL}}_{t,k}\in\mathcal{B}(\theta, r)$ with radius $r$.
\end{assumption}
\begin{assumption}[Slow generalization differ]
\label{ass:Slow generalization differ}
We assume that the semantic-shifted environment does not change arbitrarily the gradient of energy at each training step:
\begin{equation*}
\mathbb{E}_{x \sim P^{\text{sem}}_{t+1}}[\nabla_\theta E_{\theta_{t,k}}(x)] - \mathbb{E}_{x \sim P^{\text{sem}}_t}[\nabla_\theta E_{\theta_{t,k}}(x)] =\mathcal{O}(\xi_{t,k}).
\end{equation*}
\end{assumption}
\subsection{RL-GD Loss Gap Bounds}
We next analyze the training gap between GD and our RL optimizers. We recall Def.~\ref{loss-gap} and show that the loss gap is bounded by the average gap of RL agent state during training steps. 
\begin{theorem}[Loss gap bound]
\label{thm:loss-gap bound} Suppose Assumptions~\ref{ass:Lipschitz} and \ref{ass:Compactness} hold true. We bound $F^{OW}_t$ by
\begin{align}\label{F-OW:bound}
F^{OW}_t\leq \frac{1}{2} \lambda_{max} \big[ \eta_{K,r}+\eta_{RL}\sum_{k=1}^K \sum_{k'=1}^{k-1} \|s_{t,k'}-s_{t,k'-1}\|^2\big], 
\end{align}
where $\eta_{\text{RL}}$ is constant and sufficiently small, $s_{t,k}$ is the state of RL agent at time step $t$ and training step $k$.  $\eta_{K,r}$ is a constant in terms of total training steps $K$ and radius $r$, and  $\lambda_{max}$ is the maximum eigenvalue of the Hessian matrix of loss at $\theta^{\text{GD}}_{t,*}$, $\nabla^2 \mathcal{L}_{OW}(\theta^{\text{GD}}_{t,*})$. 
\end{theorem}
\begin{proof}[Proof sketch] To establish the upper bound in (\ref{F-OW:bound}), we use a multi-step argument: (1) applying second-order Taylor expansion, we approximate $\mathcal{L}_{OW}(\theta_{t,*}^{\text{RL}})$ around $\theta^{\text{GD}}_{t,*}$, (2) we bound the Hessian matrix of loss by the maximum eigenvalue eigenvalue$\lambda_{max}$, (3) we recursively use training updates for GD and RL-guided optimizers in (\ref{eq:gd_update_method}) and (\ref{eq:rl_update_method}), (4) we extend the gradient of value function $\|\nabla_\theta V(s_{t,k-1})\|^2$ in terms of gradient of value function gap between training steps from the initialization; and finally (5) we yield the result after applying Assumptions~\ref{ass:Lipschitz} and \ref{ass:Compactness}.  
\end{proof}
\begin{corollary} [Gradient of energy bounds loss gap] \label{corollary:0} Under the Assumptions~\ref{ass:nondeg}, and \ref{ass:Lipschitz}- \ref{ass:Slow generalization differ} 
we bound $F^{OW}_t$ in terms of the gradient of energy function:
\begin{align}\label{eq:corollary:0}
F_t^{OW}\leq \frac{1}{2} \lambda_{max} \big[ \mathcal{O}({\eta}_{K,r} \xi_{t})+\bar{\eta}_{RL}\sum_{k=1}^K\|\mathbb{E}_{x \sim P^{\text{sem}}_t}[\nabla_\theta E_{\theta_{t,k-1}}(x)]\|^2\big], 
\end{align}
where $\xi_t$ is the rage of state $s_t$ gradient, $\bar{\eta}_{RL}$ is a constant and sufficiently small, and $\eta_{K,r}$ is a constant in terms of total number of training steps $K$ and radius $r$. 
\end{corollary} 
\begin{corollary}[Training step-based gradient energy state bounds loss gap]\label{corollary:1}
The $F_t^{OW}$ bound in Theorem~\ref{thm:loss-gap bound} is derived in terms of training step-based gradient energy state in Def.~\ref{def:epoch-based energy gradient state}:
\begin{align}\label{F-OW:bound2}
F^{OW}_t\leq \frac{1}{2} \lambda_{max} \big[ \eta_{K,r}+\eta_{RL}\sum_{k=1}^K \sum_{k'=1}^{k-1} \| S^{(k')}_{t,\nabla}\|^2\big], 
\end{align}
where $\eta_{\text{RL}}$ is constant, $\eta_{K,r}$ is a constant in terms of total training steps $K$ and radius $r$, and  $\lambda_{max}$ is the maximum eigenvalue of $\nabla^2 \mathcal{L}_{OW}(\theta^{\text{GD}}_{t,*})$. 
\end{corollary}

\subsection{RL Improves the Effective Gradient Flow}
Theorem~\ref{thm:loss-gap bound} develops the loss gap for RL optimizer compared to GD optimizer. We next study environment and training effective gradient flow (EGF) and how our RL-guided optimizer improves the EGF across environment change and during training steps.  
\begin{definition}[Effective gradient flow]\label{def:EGF}
We define two versions of effective gradient flow (EGF):

$\bullet$ (Gradient flow (GF) over environment change) measures optimization dynamics and is typically approximated using the norm of the gradients of the loss over time. Define the matrix of 
 $\mathcal{G}_T:= [\nabla_\theta \mathcal{L}_{G}(\theta_{1})\;\nabla_\theta \mathcal{L}_{G}(\theta_2)\; \ldots\; \nabla_\theta \mathcal{L}_{G}(\theta_T)]=[\mathbf{G}_1\mathbf{G}_2\ldots \mathbf{G}_T]$
 for $T$ environments $\mathcal{E}=\{\mathcal{E}_1,\ldots,\mathcal{E}_T\}$. 
 Note that 
 $\mathcal{L}_G=\mathcal{L}_{OW}$ (for GD optimizer) and $\mathcal{L}_G=\mathcal{L}_{OW}+\eta_{RL}V(s)$ (for RL optimizer). 
 The environment-effective gradient flow for $t=1,\ldots, T$ is defined as 
  $$
  EGF^{env}_{p}:=\frac{1}{T}\sum_{t=1}^T \|\mathbf{G}_t\|_p, \;\;\hbox{where} \;\;\mathbf{G}_t=\nabla_\theta \mathcal{L}_{G}(\theta_t).
  $$
$\bullet$ (Gradient flow (GF) over training steps) measures optimization per training update and is typically approximated using the norm of the gradients of the loss over epochs at each time step. Define the matrix of 
$\mathcal{G}_K:= [\nabla_{\theta} \mathcal{L}_{G}(\theta_{t,1})\;\nabla_{\theta} \mathcal{L}_{G}(\theta_{t,2})\; \ldots\; \nabla_{\theta} \mathcal{L}_{G}(\theta_{t,K})]=[\mathbf{G}_1\mathbf{G}_2\ldots \mathbf{G}_K]$
  for $K$ training updates. The training-effective gradient flow at time $t$ is defined as $EGF^{train}_{t,p}:=\frac{1}{K}\sum_{k=1}^K \|\mathbf{G}_k\|_p$, where $\mathbf{G}_k=\nabla_{\theta} \mathcal{L}_{G}(\theta_{t,k})$. 
\end{definition}
Consider three consecutive environments $\mathcal{E}_{t-1}, \mathcal{E}_{t}$, and $\mathcal{E}_{t+1}$. The model $f_\theta$ has been trained using RL/GD optimizer, with energy values $E_{\theta_{t-1}}$, $E_{\theta_t}$, and $E_{\theta_{t+1}}$ (resp.). To evaluate the RL and GD optimizers, we introduce two new query-based notions: 
\textit{(1) Energy generalization transferability.}
Is the gap in $EGF^{env}_p$ between RL and GD optimizers indicative of energy generalization transfer across environments? Here, energy generalization transferability refers to evaluating $E_{\theta_{t-1}}$ and $E_{\theta_t}$ on semantically shifted environments $\mathcal{E}^{sem}_{t}$ and $\mathcal{E}^{sem}_{t+1}$, drawn from distributions ${P}^{sem}_{t}$ and ${P}^{sem}_{t+1}$ (resp.). \textit{(2) Semantic-shifted energy gap.} Is the gap in $EGF^{env}_p$ between RL and GD optimizers related to the energy gap induced by semantic shifts across environments? In this case, $E_{\theta_{t}}$ and $E_{\theta_{t+1}}$ are evaluated on the semantic-shifted environments $\mathcal{E}^{sem}_{t}$ and $\mathcal{E}^{sem}_{t+1}$ with distributions ${P}^{sem}_{t}$ and ${P}^{sem}_{t+1}$ (resp.). Before investigating these queries (Thms~\ref{env-EGF-Gap}, \ref{training-EGF-Gap}) we introduce the definition below. 
\begin{definition}\label{def:env-based energy state} 
    At time steps $t$ and $t-1$, we define the environment-transfer generalization energy, $S_{t\rightarrow t+1}$, by $\mathbb E_{x\sim {P}^{sem}_{t+1}}\left[E_{\theta_t}(x)\right]-\mathbb E_{x\sim {P}^{sem}_{t}}\left[E_{\theta_{t-1}}(x)\right].$
And we define the environment-based semantic shifted energy bridge, $E_{t\rightarrow t+1}$, by $\mathbb E_{x\sim {P}^{sem}_{t+1}}\left[E_{\theta_{t+1}}(x)\right]-\mathbb E_{x\sim {P}^{sem}_{t}}\left[E_{\theta_{t}}(x)\right]$. 
\end{definition}
\begin{assumption}[Slow environment-transfer generalization energy]
\label{ass:Slow environment-transfer generalization energy}
We assume that the environment does not transfer energy arbitrarily fast that is $S_{t\rightarrow t+1}=\mathcal{O}(\xi^S_{t\rightarrow t+1})$. 
\end{assumption}
\begin{assumption}[Slow environment-based semantic shifted energy]
\label{ass:Slow environment-based semantic shifted energy}
We assume that the environment does not change semantic shifted energy arbitrarily fast that is $E_{t\rightarrow t+1}=\mathcal{O}(\xi^E_{t\rightarrow t+1})$.
\end{assumption}

\begin{theorem}[Training effective gradient flow gap bound]\label{training-EGF-Gap}
Consider training-effective gradient flow $EGF^{train}_{t,p}$ under both GD and RL optimizers. Then under Assumptions~\ref{ass:Lipschitz} and\ref{ass:Compactness}, we have
\begin{align}\label{Eq:training-EGF-Gap}
EGF^{train}_{t,2}(RL)-EGF^{train}_{t,2}(GD)\leq L_r+L_{K-RL}\sum_{k=1}^K\sum_{k'=1}^k \| S^{(k')}_{t,\nabla}\|_2,
\end{align}
where $L_r$ and $L_{K-RL}$ are constants and $L_r=2\sqrt{L} r$, $L_{K-RL}=\frac{\eta_{RL}\sqrt{L_V}}{K}$. Here $S^{(k')}_{t,\nabla}$ is the training step-based gradient energy state at training step $k'$ introduced in Def.~\ref{def:epoch-based energy gradient state}.  
\end{theorem}

\begin{theorem}[Environment effective gradient flow gap bound]\label{env-EGF-Gap} Consider environment-effective gradient flow $EGF^{env}_{p}$ under both GD and RL optimizers.
Under the Assumptions~\ref{ass:Lipschitz},\ref{ass:Compactness}, \ref{ass:Slow environment-transfer generalization energy}, and \ref{ass:Slow environment-based semantic shifted energy},
we bound environment-effective gradient flow gap between RL and GD optimizers as follows: 
\begin{align}\label{eq:env-EGF-Gap}
EGF^{env}_{p}(RL)-EGF^{env}_{p}(GD) \leq L_r+ L_{T-RL}\sum_{t=1}^T \sum_{\tau=1}^t \left(\mathcal{O}(\xi^S_{\tau\rightarrow \tau+1})+\mathcal{O}(\xi^E_{\tau-1\rightarrow \tau})\right),
\end{align}
for $p=2$. In (\ref{eq:env-EGF-Gap}), $L_r$ and $L_{T-RL}$ are constants and $L_{T-RL}=\frac{\eta_{RL} \sqrt{L_V}}{T}$, for constant $L_V$ and $\eta_{RL}$.\\ 
 $\xi^S_{r\rightarrow r+1}$ and $\xi^E_{r-1\rightarrow r}$ are environment-transfer generalization energy and environment-based semantic shifted energy rates in Def.~\ref{def:env-based energy state}.
\end{theorem}
We defer the proofs of Corollary~\ref{corollary:0}, Corollary~\ref{corollary:1}, and Theorems~\ref{training-EGF-Gap} and \ref{env-EGF-Gap} to Appendix~\ref{app:training-regims_proofs}.

\noindent{\bf Discussion.} Theorem~\ref{env-EGF-Gap} provides the central theoretical foundation for the gradient flow behavior of our RL-guided optimization method. Specifically, it shows that the RL-based correction enhances gradient flow effectiveness compared to standard gradient descent. This improvement arises from the aggregated contributions of environment-transfer generalization energy drift and environment-specific semantic shift energy drift across all environments.


%% file: 5_conclusion.tex
\section{Conclusion}
We propose a novel RL-guided optimization framework for OOD detection in dynamic environments. Our RL-based detector is structured based on a TD-trained value function and lower FPR award. Our theoretical analysis explains that (i) the RL value gradient aligns with future-domain improvement, (ii) RL-guided detector tightens the model-change generalization error and implies a lower semantic-OOD FPR, (iii) RL-guided optimizer improves OOD rejection, (iv) our RL-based training optimal decreases the domain-specific loss function, and finally (v) RL-guided training mechanism improves the effective gradient flow over environment change and training steps. \\
\noindent\textbf{Limitation and Future Work}. An intriguing direction for future work would be to relax some of the assumptions. Besides implementation of our RL optimizer for OOD detection on standard OOD benchmarks, an ideal RL pipeline for domain generalization and OOD rejection in open-world to automatically identify unsafe OOD examples would be an important next step for our work. \\
\noindent\textbf{Societal Impact.} Unlike controlled environments, social sector data in real-world is messy, constantly changing, and inherently complex. Without OOD detection, models may make overconfident, incorrect, or discriminatory predictions when facing these unknown scenarios. For example, a model trained on successful healing surgical wounds (ID) for cancer patients fails to accurately diagnose rare severe adverse wound events such as excessive necrosis resulting in major amputation (OOD). 

\noindent{\bf Acknowledgement.}  Salimeh Sekeh has been partially supported by NSF CAREER CCF-2451457 and Xin Zhang is supported in part by NSF grant IIS-2449864. The findings are those of the authors only and do not represent any position of these funding bodies.

%% file: 6_appendix.tex
\appendix
\section*{Appendix}

\section{Algorithm}\label{app:algorithm}
Algorithm~\ref{alg:training} summarizes the overall training procedure. The outer loop indexes the evolving environment, while the inner loop performs detector adaptation within each environment. 
At each outer step $t$, we first calibrate the threshold $\tau_t$ from the current ID batch. We then sample labeled ID data and unlabeled wild data, construct the pseudo semantic-OOD partition using the current detector, and compute the OOD-aware objective. The state is estimated, the value function is updated, and finally the detector parameters are updated using the RL-guided rule.

\begin{algorithm}[h!]
\caption{RL-guided dynamic OOD adaptation}
\label{alg:training}
\begin{algorithmic}[1]
\STATE Initialize detector parameters $\theta_0$ and value parameters $\phi$
\FOR{$t=0,\dots,T-1$}
    \STATE Receive data stream from current environment $P_t^{\mathrm{wild}}$
    \STATE Sample a reference ID batch and calibrate $\tau_t$ from a fixed percentile of its energy distribution
    \STATE Set $\theta_{t,0}\leftarrow \theta_t$
    \FOR{$k=0,\dots,K-1$}
        \STATE Sample a batch of labeled ID data and unlabeled wild data from the current environment
        \STATE Use the current energy detector $E_{\theta_{t,k}}(\cdot)$ and threshold $\tau_t$ to form a pseudo semantic-OOD subset and a residual wild subset
        \STATE Compute the OOD-aware loss $\mathcal{L}_{OW}(\theta_{t,k})$
        \STATE Estimate the state $s_{t,k}$ from current batch statistics
        \STATE Compute the reward $r_{t,k}$
        \STATE Update value parameters $\phi$ by minimizing the squared TD error.
        \STATE Update detector parameters using \eqref{eq:rl_update_method}
    \ENDFOR
    \STATE Set $\theta_{t+1}\leftarrow \theta_{t,K}$
\ENDFOR
\end{algorithmic}
\end{algorithm}

We update the value network before the detector within each inner step so that the detector update uses the most recent estimate of the value function while avoiding a same-step circular dependency between $\phi$ and $\theta$. Relative to standard gradient descent, the additional computation in our method comes from evaluating $\nabla_\theta V_\phi(s_{t,k})$, which requires one forward-backward pass through the lightweight value network and differentiation through the state statistics. Since $s_{t,k}$ is low-dimensional and the value network is shallow, this overhead is modest compared with the detector update itself.

\section{Extended Related Work}\label{app:related_work}
\noindent\textbf{Static OOD detection.}
A large body of work studies OOD detection under a fixed deployment distribution. 
Post-hoc scoring methods assign anomaly scores using maximum softmax probability~\citep{hendrycks2016baseline}, input perturbation~\citep{liang2018enhancing}, Mahalanobis distance~\citep{lee2018simple}, and energy-based scores~\citep{liu2020energy}. 
Additional improvements exploit feature clipping~\citep{sun2021react} and contrastive learning~\citep{tack2020csi}. 
Training-time approaches further improve separation between ID and OOD samples through outlier exposure~\citep{hendrycks2018deep}. 
These methods are designed for \emph{static} settings and do not address how OOD detectors should evolve under non-stationary environments. 
In contrast, our work studies how \emph{optimization dynamics} affect OOD robustness over time.

\noindent\textbf{Test-time adaptation.}
Test-time adaptation (TTA) methods update model parameters using unlabeled test data at deployment. 
Tent~\citep{wang2020tent} minimizes prediction entropy through batch normalization, while EATA~\citep{niu2022efficient} improves stability via sample selection and regularization. 
Related approaches also explore calibration and adaptation strategies under distribution shift through similar objectives~\citep{wang2020tent, niu2022efficient}. 
While these methods improve robustness to non-stationarity, they optimize \emph{instantaneous objectives} and do not explicitly model how current updates affect future-domain OOD performance. 
Our framework instead introduces a temporally-aware objective that explicitly accounts for the long-term impact of each update.

\noindent\textbf{Wild-data and open-world OOD detection.}
Recent work explores OOD detection using unlabeled data that mixes ID and OOD samples. 
Outlier Exposure~\citep{hendrycks2018deep} demonstrates that auxiliary outlier data can improve detection performance, 
while SCONE~\citep{bai2023feed} shows that wild data can simultaneously benefit OOD detection and generalization. 
These approaches operate in a \emph{batch setting} where each update is independent. 
In contrast, we consider a \emph{temporal setting} in which data distributions evolve over time, and we analyze how the \emph{trajectory of updates} influences long-run OOD robustness.

\noindent\textbf{Continual and online adaptation under distribution shift.}
A related line of work studies model adaptation under non-stationary environments, including test-time training with self supervision~\citep{sun2020test}, entropy-based adaptation~\citep{wang2020tent}, and stability-aware updates~\citep{niu2022efficient}. 
These methods address online distribution shift by improving predictive performance at each step, but they do not explicitly analyze how optimization updates influence \emph{future-domain} OOD detection behavior. 
Our work complements this literature by providing a theoretical analysis of optimization dynamics and their role in temporal robustness.

\noindent\textbf{Continual learning and forgetting.}
Continual learning methods~\citep{kirkpatrick2017overcoming, lopez2017gradient, farajtabar2020orthogonal} study how models retain knowledge across sequential tasks. 
Our temporal error decomposition is structurally related to the stability-plasticity trade-off in this literature. 
However, our focus differs in two key aspects: 
\textit{i)} we study \emph{OOD detection} rather than classification accuracy, and \textit{ii)} we analyze how \emph{optimization dynamics} shape energy-based decision boundaries under evolving distributions.

\noindent\textbf{RL for optimization and meta-learning.}
RL and meta-learning have been used to design learned optimizers and adaptive update rules~\citep{andrychowicz2016learning, metz2019understanding}. 
Algorithm Distillation~\citep{laskin2022context} and MAML~\citep{finn2017model} learn update strategies or initializations that generalize across tasks. 
Unlike these general-purpose approaches, our framework introduces a \emph{task-structured value function} tied specifically to semantic-OOD false positive rate. 
This structure enables a direct connection between optimization dynamics and OOD detection performance, allowing us to derive explicit theoretical guarantees on future-domain robustness.

\noindent\textbf{Energy-based OOD theory.}
Energy-based formulations interpret classifiers as defining energy functions that separate ID and OOD samples~\citep{liu2020energy, duvenaud2020your}. 
Prior work provides theoretical perspectives on OOD detection under distribution shift by analyzing the role of energy functions and classifier
geometry~\citep{liu2020energy}. 
Our work builds on this foundation and extends it to the \emph{dynamic setting}, providing, to our knowledge, the first optimization-theoretic analysis connecting update dynamics to temporal OOD robustness through energy separation.

\section{Assumption Discussion}\label{app:assumptions}

Our analysis relies on a combination of standard regularity conditions and problem-specific structural assumptions. 
Assumptions \ref{ass:smooth}, \ref{ass:bounded-grad}, and \ref{ass:nondeg} are typical in first-order optimization analyses: 
they ensure that the loss and energy functions are locally smooth, gradients remain bounded, and the energy landscape provides a non-degenerate signal for semantic-OOD samples. 
These conditions are generally satisfied along the training trajectory under small step sizes and well-behaved model parameterizations.

The remaining assumptions characterize the dynamic OOD and training-related settings: 

Assumptions \ref{ass:energy-fpr} and \ref{ass:reward-mono} connect the value function to OOD detection behavior, ensuring that improvements in semantic-OOD energy translate into lower FPR.
Assumption~\ref{ass:grad-conflict} captures a structural tension between classification and OOD separation, which we show is realizable in practical models (Sec.~\ref{subsec:theory_transformer_example}).
Finally, Assumptions~\ref{ass:slow-dist} and~\ref{ass:slow-gerr} define a gradual-drift regime in which the environment evolves at a comparable scale to the RL correction, allowing one-step improvements to accumulate over time rather than being dominated by distributional changes. 

Assumptions \ref{ass:Lipschitz}-\ref{ass:Slow generalization differ}, controls the optimization updates and training drifts by applying smooth gradient, bounding the stochastic changes of parameters, and slow gradient of energy function under semantic-shifted distribution. Assumptions~\ref{ass:Slow environment-transfer generalization energy} and~\ref{ass:Slow environment-based semantic shifted energy} characterize energy drift under environment changes, ensuring that both environment-transfer generalization and environment-based semantic shift energies evolve gradually over time.

\section{Proofs for Section~\ref{sec:theory}}
\label{app:theory_proofs}

In this section, we provide detailed proofs of the main theoretical results in Section~\ref{sec:theory}. 
Throughout, we use the same notation as in the main text. 
In particular, $\theta_{t,k}$ denotes the detector parameter at outer step $t$ and inner step $k$, and for the purpose of analysis we use the semantic-drift summary as a reduced state representation,
\begin{equation}
\label{eq:app_z_tk_def}
z_{t,k} :=
\mathbb{E}_{x\sim P_{t+1}^{\rm sem}}[E_{\theta_{t,k}}(x)]-
\mathbb{E}_{x\sim P_t^{\rm sem}}[E_{\theta_{t,k}}(x)].
\end{equation}
In the proofs below, we use $z_{t,k}$ and $s_{t,k}$ interchangeably. 

\subsection{Proof of Lemma~\ref{lem:alignment}}
\label{app:proof_alignment}

\begin{proof}
Fix an outer step $t$ and an inner step $k$, and write
$\theta := \theta_{t,k}$ for brevity.
The proof isolates the two components of the state $s_{t,k}$ that are relevant to the argument: 
the semantic-drift coordinate $z_{t,k}$ and the FPR ${\rm FPR}_{t,k}$. 
By the chain rule,
\begin{equation}
\label{eq:app_alignment_chain}
\nabla_\theta V(s_{t,k})
=\frac{\partial V}{\partial z_{t,k}}\,\nabla_\theta z_{t,k} +
\frac{\partial V}{\partial {\rm FPR}_{t,k}}\,\nabla_\theta {\rm FPR}_{t,k}.
\end{equation}
We control the two terms separately.

\noindent{\it Step 1: the semantic-drift component is small.}
Differentiating \Eqref{eq:app_z_tk_def} with respect to $\theta$ gives
\begin{equation*}
\label{eq:app_grad_z}
\nabla_\theta z_{t,k} =
\mathbb{E}_{x\sim P_{t+1}^{\rm sem}}[\nabla_\theta E_\theta(x)] -
\mathbb{E}_{x\sim P_t^{\rm sem}}[\nabla_\theta E_\theta(x)].
\end{equation*}
By Assumption~\ref{ass:bounded-grad}, $\|\nabla_\theta E_\theta(x)\| \le G_E$ uniformly in $x$. 
Therefore, using the standard total-variation inequality for bounded vector-valued
functions,
$\|\nabla_\theta z_{t,k}\| \le 2G_E\,d_{\rm TV}(P_{t+1}^{\rm sem},P_t^{\rm sem})$.
Applying Assumption~\ref{ass:slow-dist},
\begin{equation}
\label{eq:app_grad_z_bound}
\|\nabla_\theta z_{t,k}\| \le
2G_E\,\xi_t = \mathcal{O}(\xi_t).
\end{equation}
Since the partial derivative $\frac{\partial V}{\partial z_{t,k}}$ is uniformly bounded by Assumption~\ref{ass:reward-mono}, 
the semantic-drift contribution in \Eqref{eq:app_alignment_chain} is $\mathcal{O}(\xi_t)$.

\noindent{\it Step 2: expression for the FPR gradient.}
By differentiating the threshold event through the energy score and using
the density $p(\tau)>0$ at the threshold (Assumption~\ref{ass:nondeg}),
we obtain
\begin{equation}
\label{eq:app_grad_fpr}
\nabla_\theta {\rm FPR}_{t,k} = -p(\tau)\,
\mathbb{E}_{x\sim P_t^{\rm sem}}[\nabla_\theta E_\theta(x)].
\end{equation}

\noindent{\it Step 3: the FPR component has the correct sign.}
Taking the inner product of \Eqref{eq:app_grad_fpr} with
$\nabla_\theta {\rm GErr}_{t+1}(\theta)$ yields
\begin{align}
\label{eq:app_alignment_fpr_start}
\bigl\langle
\nabla_\theta {\rm GErr}_{t+1}(\theta),
\nabla_\theta {\rm FPR}_{t,k}
\bigr\rangle &=
-p(\tau)\,\Bigl\langle\nabla_\theta {\rm GErr}_{t+1}(\theta),
\mathbb{E}_{x\sim P_t^{\rm sem}}[\nabla_\theta E_\theta(x)]\Bigr\rangle.
\end{align}
Using the definition of future-domain generalization error, $\nabla_\theta {\rm GErr}_{t+1}(\theta) =
\mathbb{E}_{x'\sim P_{t+1}^{\rm cov}}
[\nabla_\theta L_{\rm CE}(f_\theta(x'))]$.
Substituting this into \Eqref{eq:app_alignment_fpr_start} and applying
Fubini's theorem gives
\begin{align}
\label{eq:app_alignment_fubini}
\bigl\langle \nabla_\theta {\rm GErr}_{t+1}(\theta), \nabla_\theta {\rm FPR}_{t,k} \bigr\rangle &=
-p(\tau)\,\mathbb{E}_{x\sim P_t^{\rm sem}}\mathbb{E}_{x'\sim P_{t+1}^{\rm cov}}
\Bigl[\bigl\langle\nabla_\theta L_{\rm CE}(f_\theta(x')),
\nabla_\theta E_\theta(x)\bigr\rangle\Bigr].
\end{align}
By Assumption~\ref{ass:grad-conflict},
$\mathbb{E}_{x'\sim P_{t+1}^{\rm cov}}
\Bigl[\bigl\langle\nabla_\theta L_{\rm CE}(f_\theta(x')),
\nabla_\theta E_\theta(x)\bigr\rangle\Bigr]
\le
-\rho_{\min}\|\nabla_\theta E_\theta(x)\|$.
Plugging this into \Eqref{eq:app_alignment_fubini} yields
\begin{equation}
\label{eq:app_alignment_positive}
\bigl\langle\nabla_\theta {\rm GErr}_{t+1}(\theta),
\nabla_\theta {\rm FPR}_{t,k}\bigr\rangle\ge
p(\tau)\rho_{\min}\,
\mathbb{E}_{x\sim P_t^{\rm sem}}
[\|\nabla_\theta E_\theta(x)\|].
\end{equation}

\noindent\textit{Step 4: combine the two components.}
Combining \Eqref{eq:app_alignment_chain},
\Eqref{eq:app_grad_z_bound}, and \Eqref{eq:app_alignment_positive}, and
using $\frac{\partial V}{\partial {\rm FPR}_{t,k}}\le -c<0$ from Assumption~\ref{ass:reward-mono}, we obtain
\begin{align*}
\bigl\langle \nabla_\theta {\rm GErr}_{t+1}(\theta),
\nabla_\theta V(s_{t,k}) \bigr\rangle
&=
\frac{\partial V}{\partial z_{t,k}}\bigl\langle
\nabla_\theta {\rm GErr}_{t+1}(\theta), \nabla_\theta z_{t,k}\bigr\rangle
+
\frac{\partial V}{\partial {\rm FPR}_{t,k}}\bigl\langle
\nabla_\theta {\rm GErr}_{t+1}(\theta),\nabla_\theta {\rm FPR}_{t,k}
\bigr\rangle \\
&\le
\mathcal{O}(\xi_t)-c\,p(\tau)\rho_{\min}\,\mathbb{E}_{x\sim P_t^{\rm sem}}
[\|\nabla_\theta E_\theta(x)\|].
\end{align*}
Define $\kappa(\theta_{t,k}):=c\,p(\tau)\,\rho_{\min}\,
\mathbb{E}_{x\sim P_t^{\rm sem}}[\|\nabla_\theta E_{\theta_{t,k}}(x)\|]$.
Then
$\bigl\langle\nabla_\theta {\rm GErr}_{t+1}(\theta_{t,k}),\nabla_\theta V(s_{t,k})\bigr\rangle\le
-\kappa(\theta_{t,k}) + \mathcal{O}(\xi_t)$.
This proves the result.
\end{proof}

\subsection{Proof of Lemma~\ref{lem:model_change}}
\label{app:proof_model_change}

\begin{proof}
Recall
$\Delta_t^{\rm GD}:=
{\rm GErr}_{t+1}(\theta_{t,K}^{\rm GD})-
{\rm GErr}_{t+1}(\theta_{t,0})$, 
and $\Delta_t^{\rm RL}:={\rm GErr}_{t+1}(\theta_{t,K}^{\rm RL})-{\rm GErr}_{t+1}(\theta_{t,0})$,
where both optimizers share the same initialization $\theta_{t,0}$.

By Assumption~\ref{ass:smooth}, ${\rm GErr}_{t+1}$ is $\beta$-smooth.
Expanding around $\theta_{t,K}^{\rm GD}$ yields
\begin{align}
\label{eq:app_model_smooth}
\Delta_t^{\rm RL} &\le \Delta_t^{\rm GD}
+\bigl\langle
\nabla_\theta {\rm GErr}_{t+1}(\theta_{t,K}^{\rm GD}),
\theta_{t,K}^{\rm RL}-\theta_{t,K}^{\rm GD}
\bigr\rangle
+
\frac{\beta}{2}\|\theta_{t,K}^{\rm RL}-\theta_{t,K}^{\rm GD}\|^2.
\end{align}
We next analyze the parameter gap
$\theta_{t,K}^{\rm RL}-\theta_{t,K}^{\rm GD}$.

\noindent\textit{Step 1: parameter-gap decomposition.}
From the update rules, $\theta_{t,k+1}^{\rm GD}=
\theta_{t,k}^{\rm GD}-\eta \nabla_\theta \mathcal{L}_{OW}(\theta_{t,k}^{\rm GD})$,
and $\theta_{t,k+1}^{\rm RL} = \theta_{t,k}^{\rm RL}
- \eta \nabla_\theta \mathcal{L}_{OW}(\theta_{t,k}^{\rm RL}) + \eta_{\rm RL}\nabla_\theta V(s_{t,k})$.
Subtracting the two recursions and summing over $k=0,\dots,K-1$ gives
\begin{align}
\label{eq:app_model_gap_decomp}
\theta_{t,K}^{\rm RL}-\theta_{t,K}^{\rm GD}
=\eta_{\rm RL}\sum_{k=0}^{K-1}\nabla_\theta V(s_{t,k})
-
\eta\sum_{k=0}^{K-1}
\Bigl(
\nabla_\theta \mathcal{L}_{OW}(\theta_{t,k}^{\rm RL})
-
\nabla_\theta \mathcal{L}_{OW}(\theta_{t,k}^{\rm GD})
\Bigr).
\end{align}

\noindent\textit{Step 2: recursion bound.}
Let $d_k := \|\theta_{t,k}^{\rm RL}-\theta_{t,k}^{\rm GD}\|$.
Since $\nabla_\theta \mathcal{L}_{OW}$ is $\beta$-Lipschitz by Assumption~\ref{ass:smooth},
$d_{k+1} \le (1+\eta\beta)d_k + \eta_{\rm RL}G_V$.
With $d_0=0$, a standard induction yields $d_k = \mathcal{O}(k\eta_{\rm RL})$.
Therefore,
$\eta\sum_{k=0}^{K-1}
\Bigl\| \nabla_\theta \mathcal{L}_{OW}(\theta_{t,k}^{\rm RL}) -
\nabla_\theta \mathcal{L}_{OW}(\theta_{t,k}^{\rm GD})
\Bigr\|=
\mathcal{O}(K^2\eta\,\eta_{\rm RL})$.
Substituting into \Eqref{eq:app_model_gap_decomp},
\begin{equation}
\label{eq:app_model_gap_final}
\theta_{t,K}^{\rm RL}-\theta_{t,K}^{\rm GD}
=\eta_{\rm RL}\sum_{k=0}^{K-1}\nabla_\theta V(s_{t,k})+
\mathcal{O}(K^2\eta\,\eta_{\rm RL}).
\end{equation}

\noindent\textit{Step 3: substitute into the smoothness bound.}
Plugging \Eqref{eq:app_model_gap_final} into
\Eqref{eq:app_model_smooth} gives
\begin{align*}
\Delta_t^{\rm RL} &\le
\Delta_t^{\rm GD} + \eta_{\rm RL} \sum_{k=0}^{K-1}
\bigl\langle
\nabla_\theta {\rm GErr}_{t+1}(\theta_{t,K}^{\rm GD}),
\nabla_\theta V(s_{t,k})
\bigr\rangle
+
\mathcal{O}(K^2\eta\,\eta_{\rm RL}),
\end{align*}
where the quadratic term is absorbed into the same remainder because
\(\|\theta_{t,K}^{\rm RL}-\theta_{t,K}^{\rm GD}\|=O(K\eta_{\rm RL})\)
and \(\eta_{\rm RL}\le \eta\).

It remains to shift the gradient evaluation point from
$\theta_{t,K}^{\rm GD}$ to $\theta_{t,k}^{\rm GD}$. Specifically, by $\beta$-smoothness of ${\rm GErr}_{t+1}$ and Assumption~\ref{ass:bounded-grad},
$\bigl| \langle
\nabla_\theta {\rm GErr}_{t+1}(\theta_{t,K}^{\rm GD})
- \nabla_\theta {\rm GErr}_{t+1}(\theta_{t,k}^{\rm GD}), \nabla_\theta V(s_{t,k})\rangle\bigr|
\le
\beta \|\theta_{t,K}^{\rm GD}-\theta_{t,k}^{\rm GD}\|\,G_V$.
Moreover, each GD step moves by at most $\eta G$, so
$\|\theta_{t,K}^{\rm GD}-\theta_{t,k}^{\rm GD}\|
\le (K-k)\eta G$.
Multiplying by $\eta_{\rm RL}$ and summing over $k$ yields an additional error of order
$\eta_{\rm RL}\sum_{k=0}^{K-1} O((K-k)\eta)
=\mathcal{O}(K^2\eta\,\eta_{\rm RL})$,
which is absorbed into the remainder. 
Therefore,
$\Delta_t^{\rm RL} \le
\Delta_t^{\rm GD} +
\eta_{\rm RL} \sum_{k=0}^{K-1}
\bigl\langle
\nabla_\theta {\rm GErr}_{t+1}(\theta_{t,k}^{\rm GD}),
\nabla_\theta V(s_{t,k})
\bigr\rangle + \mathcal{O}(K^2\eta\,\eta_{\rm RL})$.
Applying Lemma~\ref{lem:alignment} at each inner step yields $\Delta_t^{\rm RL} \le
\Delta_t^{\rm GD} - \eta_{\rm RL}\sum_{k=0}^{K-1}\kappa(\theta_{t,k}) +
\mathcal{O}(K^2\eta\,\eta_{\rm RL})$,
as claimed.
\end{proof}

\subsection{Proof of Lemma~\ref{lem:energy_bridge}}
\label{app:proof_energy_bridge}

\begin{proof}
Define the expected semantic-OOD energy at outer step $t$ by $\bar E_t(\theta) := \mathbb{E}_{x\sim P_t^{\rm sem}}[E_\theta(x)]$.
Since $E_\theta(x)$ is $\beta_E$-smooth in $\theta$ by
Assumption~\ref{ass:nondeg}, the expectation $\bar E_t(\theta)$ is also $\beta_E$-smooth. 
Therefore, a first-order expansion around
$\theta_{t,K}^{\rm GD}$ gives
\begin{align*}
\bar E_t(\theta_{t,K}^{\rm RL})-\bar E_t(\theta_{t,K}^{\rm GD}) &=
\Bigl\langle
\mathbb{E}_{x\sim P_t^{\rm sem}}
[\nabla_\theta E_{\theta_{t,K}^{\rm GD}}(x)],
\theta_{t,K}^{\rm RL}-\theta_{t,K}^{\rm GD}
\Bigr\rangle
+ \mathcal{O}(\|\theta_{t,K}^{\rm RL}-\theta_{t,K}^{\rm GD}\|^2).
\end{align*}

Using Lemma~\ref{lem:model_change},
$\theta_{t,K}^{\rm RL}-\theta_{t,K}^{\rm GD} =
\eta_{\rm RL}\sum_{k=0}^{K-1}\nabla_\theta V(s_{t,k}) +
\mathcal{O}(K^2\eta\,\eta_{\rm RL})$,
so, after substituting and shifting the energy-gradient evaluation point from $\theta_{t,K}^{\rm GD}$ to $\theta_{t,k}^{\rm GD}$ via $\beta_E$-smoothness, we obtain
\begin{align}
\label{eq:app_energy_reduced}
\bar E_t(\theta_{t,K}^{\rm RL})-\bar E_t(\theta_{t,K}^{\rm GD}) &=
\eta_{\rm RL}\sum_{k=0}^{K-1}
\Bigl\langle
\mathbb{E}_{x\sim P_t^{\rm sem}}
[\nabla_\theta E_{\theta_{t,k}^{\rm GD}}(x)],
\nabla_\theta V(s_{t,k})
\Bigr\rangle +
\mathcal{O}(K^2\eta\,\eta_{\rm RL}).
\end{align}
We now lower bound each inner product. 
As in the proof of
Lemma~\ref{lem:alignment}, decompose $\nabla_\theta V(s_{t,k})$ into its semantic-drift and FPR components. 
The semantic-drift component contributes only $\mathcal{O}(\xi_t)=\mathcal{O}(\eta_{\rm RL})$ by Assumption~\ref{ass:slow-dist}. 
For the FPR component,
$\nabla_\theta {\rm FPR}_{t,k} =
-p(\tau)\, \mathbb{E}_{x\sim P_t^{\rm sem}}[\nabla_\theta E_{\theta_{t,k}^{\rm GD}}(x)]$.
Hence,
\begin{align}
\label{eq:app_energy_inner}
\Bigl\langle
\mathbb{E}_{x\sim P_t^{\rm sem}}[\nabla_\theta E_{\theta_{t,k}^{\rm GD}}(x)], \nabla_\theta V(s_{t,k})
\Bigr\rangle
&=\frac{\partial V}{\partial {\rm FPR}_{t,k}} \cdot (-p(\tau))
\left\|
\mathbb{E}_{x\sim P_t^{\rm sem}}[\nabla_\theta E_{\theta_{t,k}^{\rm GD}}(x)]
\right\|^2 + \mathcal{O}(\eta_{\rm RL}).
\end{align}
Since $\frac{\partial V}{\partial {\rm FPR}_{t,k}}\le -c<0$, and by Assumption~\ref{ass:nondeg},
$\left\|\mathbb{E}_{x\sim P_t^{\rm sem}}[\nabla_\theta E_{\theta_{t,k}^{\rm GD}}(x)]\right\|\ge \mu_{\min}$,
each term in \Eqref{eq:app_energy_inner} is bounded below by $c\,p(\tau)\,\mu_{\min}^2 + \mathcal{O}(\eta_{\rm RL})$.
Summing over $k=0,\dots,K-1$ in \Eqref{eq:app_energy_reduced} yields
$\bar E_t(\theta_{t,K}^{\rm RL})-\bar E_t(\theta_{t,K}^{\rm GD}) \ge \eta_{\rm RL}Kc\,p(\tau)\,\mu_{\min}^2 - \mathcal{O}(K^2\eta\,\eta_{\rm RL})$,
which is strictly positive for sufficiently small $\eta_{\rm RL}$.

Finally, Assumption~\ref{ass:energy-fpr} implies that a larger expected semantic-OOD energy yields a lower false positive rate, hence ${\rm FPR}_t^{\rm RL}\le {\rm FPR}_t^{\rm GD}$.
\end{proof}

\subsection{Proof of Theorem~\ref{thm:main}}
\label{app:proof_main}

\begin{proof}
Applying Def.~\ref{def:temporal_decomp} at each outer step $\tau$ and summing from $\tau=0$ to $t$, we obtain a telescoping decomposition of the total difference between RL and GD:
${\rm GErr}_{t+1}^{\rm RL} - {\rm GErr}_{t+1}^{\rm GD}=
\sum_{\tau=0}^t
(\Delta_\tau^{\rm RL}-\Delta_\tau^{\rm GD})+
\sum_{\tau=0}^t
({\rm Env}_\tau^{\rm RL}-{\rm Env}_\tau^{\rm GD})$,
where ${\rm Env}_\tau^\square$ denotes the environment-change term.

\noindent\textit{Step 1: sum the model-change gains.}
By Lemma~\ref{lem:model_change}, $\Delta_\tau^{\rm RL}-\Delta_\tau^{\rm GD} \le
-\eta_{\rm RL}\sum_{k=0}^{K-1}\kappa(\theta_{\tau,k}) + \mathcal{O}(K^2\eta\,\eta_{\rm RL})$.
Summing over $\tau=0,\dots,t$,
$\sum_{\tau=0}^t(\Delta_\tau^{\rm RL}-\Delta_\tau^{\rm GD}) \le -\eta_{\rm RL}\sum_{\tau=0}^t\sum_{k=0}^{K-1}\kappa(\theta_{\tau,k}) + \mathcal{O}(tK^2\eta\,\eta_{\rm RL})$.
Since $\kappa(\theta_{\tau,k})\ge \kappa_{\min}$,
\begin{equation}
\label{eq:app_main_model}
\sum_{\tau=0}^t(\Delta_\tau^{\rm RL}-\Delta_\tau^{\rm GD}) \le -\eta_{\rm RL}\kappa_{\min}(t+1)K
+ \mathcal{O}(tK^2\eta\,\eta_{\rm RL}).
\end{equation}

\noindent\textit{Step 2: bound the accumulated environment drift.}
Define the outer-step starting-point gap
$\delta_\tau := \|\theta_{\tau,0}^{\rm RL}-\theta_{\tau,0}^{\rm GD}\|$.
Iterating the per-window parameter-gap bound from
Lemma~\ref{lem:model_change} gives $\delta_\tau = O(\tau K \eta_{\rm RL})$.
Since ${\rm GErr}$ is $L_G$-Lipschitz by Assumption~\ref{ass:smooth},
the difference between RL and GD in the environment-change term at outer step $\tau$ is controlled entirely by the starting-point gap $\delta_\tau$, giving
$\mathcal{O}(L_G\delta_\tau) = \mathcal{O}(\tau K L_G \eta_{\rm RL})$.
Thus, this step uses only the regularity assumption. 
Summing over $\tau=0,\dots,t$ yields
\begin{equation}
\label{eq:app_main_env}
\sum_{\tau=0}^t
({\rm Env}_\tau^{\rm RL}-{\rm Env}_\tau^{\rm GD}) =
\mathcal{O}(t^2 K L_G\eta_{\rm RL}).
\end{equation}

\noindent\textit{Step 3: combine the two contributions.}
Combining \Eqref{eq:app_main_model} and \Eqref{eq:app_main_env}, we obtain
${\rm GErr}_{t+1}^{\rm RL} \le
{\rm GErr}_{t+1}^{\rm GD} - \eta_{\rm RL}\kappa_{\min}(t+1)K +
\mathcal{O}(tK^2\eta\,\eta_{\rm RL}) +
\mathcal{O}(t^2 K L_G\eta_{\rm RL})$.

\noindent\textit{Step 4: FPR comparison.}
The FPR bound follows directly from Lemma~\ref{lem:energy_bridge}. 
In particular,
${\rm FPR}_t^{\rm RL} \le
{\rm FPR}_t^{\rm GD} - c\,p(\tau)\,\mu_{\min}^2 K\eta_{\rm RL} +
\mathcal{O}(K^2\eta\,\eta_{\rm RL})$,
which proves the theorem.
\end{proof}

\section{Proofs for Section~\ref{training-regims-sec}}
\label{app:training-regims_proofs}

For the purpose of analysis we use the semantic-drift summary as a reduced state representation,
$z_{t,k} :=
\mathbb{E}_{x\sim P_{t+1}^{\rm sem}}[E_{\theta_{t,k}}(x)]-
\mathbb{E}_{x\sim P_t^{\rm sem}}[E_{\theta_{t,k}}(x)]$.
In the proofs below, we use $z_{t,k}$ and $s_{t,k}$ interchangeably. 
\subsection{Proof of Theorem~\ref{thm:loss-gap bound}}
\label{app:proof_thm:loss-gap bound}
\begin{proof}
Recall the Def.~\ref{loss-gap}: 
$F^{OW}_{t}:= \mathcal{L}_{OW}(\theta_{t,*}^{\text{RL}})-\mathcal{L}_{OW}(\theta^{\text{GD}}_{t,*})$. Applying second-order Taylor expansion, we approximate $\mathcal{L}_{OW}(\theta_{t,*}^{\text{RL}})$ around $\theta^{\text{GD}}_{t,*}$ by
\begin{align}\label{eq:approx:1}
\mathcal{L}_{OW}(\theta^{\text{GD}}_{t,*})+(\theta_{t,*}^{\text{RL}}-\theta^{\text{GD}}_{t,*})^{\top} \nabla \mathcal{L}_{OW}(\theta^{\text{GD}}_{t,*})+\frac{1}{2} (\theta_{t,*}^{\text{RL}}-\theta^{\text{GD}}_{t,*})^{\top}\nabla^2 \mathcal{L}_{OW}(\theta^{\text{GD}}_{t,*})(\theta_{t,*}^{\text{RL}}-\theta^{\text{GD}}_{t,*}),
\end{align}
where, $\nabla^2 \mathcal{L}_{OW}(\theta^{\text{GD}}_{t,*})$ is the Hessian matrix for loss $\mathcal{L}_{OW}$ at point $\theta^{\text{GD}}_{t,*}$. Note that since $\theta^{\text{GD}}_{t,*}$ in an optimizer therefore $\nabla \mathcal{L}_{OW}(\theta^{\text{GD}}_{t,*})=0$. This simplifies the approximation of $\mathcal{L}_{OW}(\theta_{t,*}^{\text{RL}})$ in (\ref{eq:approx:1}) as
\begin{align*}
&\mathcal{L}_{OW}(\theta_{t,*}^{\text{RL}})\approx \mathcal{L}_{OW}(\theta^{\text{GD}}_{t,*})+\frac{1}{2} (\theta_{t,*}^{\text{RL}}-\theta^{\text{GD}}_{t,*})^{\top}\nabla^2 \mathcal{L}_{OW}(\theta^{\text{GD}}_{t,*})(\theta_{t,*}^{\text{RL}}-\theta^{\text{GD}}_{t,*}).
\end{align*}
Equivalently 
\begin{align}\label{eq:approx:2}
&F_t^{OW}:=\mathcal{L}_{OW}(\theta_{t,*}^{\text{RL}})-\mathcal{L}_{OW}(\theta^{\text{GD}}_{t,*})\approx \frac{1}{2} (\theta_{t,*}^{\text{RL}}-\theta^{\text{GD}}_{t,*})^{\top}\nabla^2 \mathcal{L}_{OW}(\theta^{\text{GD}}_{t,*})(\theta_{t,*}^{\text{RL}}-\theta^{\text{GD}}_{t,*}).
\end{align}
Let $\lambda_{max}$ be the maximum eigenvalue of $\nabla^2 \mathcal{L}_{OW}(\theta^{\text{GD}}_{t,*})$. Using (\ref{eq:approx:2}), we upper-bound $F^{OW}_t$ by
\begin{align}\label{F-OW:bound0}
F_t^{OW}\leq \frac{1}{2} \lambda_{max} \|\theta_{t,*}^{\text{RL}}-\theta^{\text{GD}}_{t,*}\|^2. 
\end{align}
Let $\theta_{t,K}^{\text{GD}}$ and $\theta_{t,K}^{\text{RL}}$ denote the parameters after $K$ training steps of standard GD and RL-augmented optimizers, respectively. Going back to the GD update training update in (\ref{eq:gd_update_method}) and RL-guided training update in (\ref{eq:rl_update_method}):
\begin{align}
    \theta_{t,K}^{\text{GD}} &= \theta^{\text{GD}}_{t,K-1} - \eta \nabla_\theta {L}_{OW}(\theta^{\text{GD}}_{t,K-1}), \label{eq:update:GD}\\
    \theta_{t,K}^{\text{RL}} &= \theta^{\text{RL}}_{t,K-1} - \eta \nabla_\theta {L}_{OW}(\theta^{\text{RL}}_{t,K-1}) + \eta_{\text{RL}} \nabla_\theta V(s_{t,K-1}). \label{eq:update:RL}
\end{align}
We suppose that after $K$ training steps , GD and RL-guided updates return the optimizers, that is $\theta_{t,*}^{\text{RL}}=\theta_{t,K}^{\text{RL}}$ and $\theta^{\text{GD}}_{t,*}= \theta_{t,K}^{\text{GD}}$, when the network is trained $K$ updates (epochs). By utilizing Eq. (\ref{eq:update:GD}) and Eq. (\ref{eq:update:RL}) and applying triangle inequality, we have 
\begin{align} \label{eq:update:GD-RL}
\|\theta_{t,*}^{\text{RL}}-\theta^{\text{GD}}_{t,*}\|^2\leq& \|\theta^{\text{RL}}_{t,K-1}-\theta_{t,K-1}^{\text{GD}}\|^2+\eta^2\| \nabla_\theta {L}_{OW}(\theta^{\text{GD}}_{t,K-1})-\nabla_\theta {L}_{OW}(\theta^{\text{RL}}_{t,K-1})\|^2\nonumber\\
&+\eta_{\text{RL}}^2 \|\nabla_\theta V(s_{t,K-1})\|^2.
\end{align}
Recursively for training updates $k=0,\ldots K$, we extend (\ref{eq:update:GD-RL}) by 
\begin{align}\label{eq:update:GD-RL:2}
\|\theta_{t,*}^{\text{RL}}-\theta^{\text{GD}}_{t,*}\|^2\leq &\|\theta^{\text{RL}}_{t,0}-\theta_{t,0}^{\text{GD}}\|^2+\eta^2\sum_{k=1}^K\| \nabla_\theta {L}_{OW}(\theta^{\text{GD}}_{t,k-1})-\nabla_\theta {L}_{OW}(\theta^{\text{RL}}_{t,k-1})\|^2\nonumber\\
&+\eta_{\text{RL}}^2 \sum_{k=1}^K \|\nabla_\theta V(s_{t,k-1})\|^2
\end{align}
Without lose of generality, we assume that the initialization points for GD and RL-guided optimizers are the same meaning that $ \|\theta^{\text{RL}}_{t,0}-\theta_{t,0}^{\text{GD}}\|^2=0$. This simplifies (\ref{eq:update:GD-RL:2}) as follows:
\begin{align}\label{F-OW:bound5}
\|\theta_{t,*}^{\text{RL}}-\theta^{\text{GD}}_{t,*}\|^2\leq \eta^2\sum_{k=1}^K\| \nabla_\theta {L}_{OW}(\theta^{\text{GD}}_{t,k-1})-\nabla_\theta {L}_{OW}(\theta^{\text{RL}}_{t,k-1})\|^2+\eta_{\text{RL}}^2 \sum_{k=1}^K \|\nabla_\theta V(s_{t,k-1})\|^2. 
\end{align}
{\bf Remark:} Note that for environment $t$,
the state is a function of parameter $\theta_t$, therefore during the training $s_{t,k}\neq s_{t,k'}$ for $k,k'\in\{0,1,\ldots K\}$, $i\neq j$. 

Going back to the proof, next, we extend the term $\|\nabla_\theta V(s_{t,k-1})\|^2$ in (\ref{F-OW:bound5}): first we add and subtract the term $\nabla_\theta V(s_{t,k-2})$,  then use triangle inequality, 
\begin{align*}
\|\nabla_\theta V(s_{t,k-1})\|^2&=\|\nabla_\theta V(s_{t,k-1})-\nabla_\theta V(s_{t,k-2})+\nabla_\theta V(s_{t,k-2})\|^2\\
&\leq \|\nabla_\theta V(s_{t,k-1})-\nabla_\theta V(s_{t,k-2})\|^2+\|\nabla_\theta V(s_{t,k-2})\|^2, 
\end{align*}
recursively we extend the above inequality for training update $k'=1,\ldots k-1$ as follows
\begin{align}\label{grad-value}
\|\nabla_\theta V(s_{t,k-1})\|^2\leq \sum_{k'=1}^{k-1} \|\nabla_\theta V(s_{t,k'})-\nabla_\theta V(s_{t,k'-1})\|^2. 
\end{align}
Note that the gradient of value function is zero at the initialization, i.e. $\nabla_\theta V(s_{t,0})=0$.

Next we first apply Assumption~\ref{ass:Lipschitz} on value function $V_\theta$ and $\mathcal{L}_{OW}$. Hence the bound in (\ref{F-OW:bound5}) is upper-bounded by
\begin{align}\label{F-OW:bound6}
\|\theta_{t,*}^{\text{RL}}-\theta^{\text{GD}}_{t,*}\|^2\leq \eta^2 L\sum_{k=1}^K\|\theta^{\text{GD}}_{t,k-1}-\theta^{\text{RL}}_{t,k-1}\|^2+\eta_{\text{RL}}^2 L_V \sum_{k=1}^K \sum_{k'=1}^{k-1} \|s_{t,k'}-s_{t, k'-1}\|^2.
\end{align}
Secondly apply Assumption~\ref{ass:Compactness}, meaning that at each training update $\|\theta^{\text{GD}}_{t,k}-\theta^{\text{RL}}_{t,k}$ is within a parameter ball with radius $r$, $\mathcal{B}(\theta, r)$. This yields to the bound below for (\ref{F-OW:bound6}),
\begin{align}\label{F-OW:bound7}
\|\theta_{t,*}^{\text{RL}}-\theta^{\text{GD}}_{t,*}\|^2\leq 4\eta^2 L\;K\; r^2 +\eta_{\text{RL}}^2 L_V \sum_{k=1}^K \sum_{k'=1}^{k-1} \|s_{t,k'}-s_{t,k'-1}\|^2
\end{align}
Finally, combining (\ref{F-OW:bound0}) and (\ref{F-OW:bound7}), we derive the our main bound in (\ref{F-OW:bound}):
\begin{align}\label{F-OW:bound8}
F_t^{OW}\leq \frac{1}{2} \lambda_{max} \left[ \eta_{K,r}+\eta_{RL}\sum_{k=1}^K \sum_{k'=1}^{k-1} \|s_{t,k'}-s_{t,k'-1}\|^2\right], 
\end{align}
where $\eta_{RL}=\eta_{RL}^2 L_V$ and $\eta_{K,r}=4\eta^2L\;K\;r^2$. 
\end{proof}

\subsection{Proof of Corollary~\ref{corollary:0}}
\label{app:corollary:0}
\begin{proof}
Recall the gradient of value $\nabla_\theta V(s_{t,k})$ from (\ref{eq:app_alignment_chain}) in Section~\ref{app:proof_alignment}:
\begin{equation}\label{eq:grad:value:1}
\nabla_\theta V(s_{t,k})
=\frac{\partial V}{\partial z_{t,k}}\,\nabla_\theta z_{t,k} +
\frac{\partial V}{\partial {\rm FPR}_{t,k}}\,\nabla_\theta {\rm FPR}_{t,k}.
\end{equation}
Next, we recall the following equality from Subsection~\ref{eq:app_grad_z} and apply Assumption~\ref{ass:Slow generalization differ}
\begin{equation}
\label{eq:app_grad_z:1}
\nabla_\theta z_{t,k} =
\mathbb{E}_{x\sim P_{t+1}^{\rm sem}}[\nabla_\theta E_\theta(x)] -
\mathbb{E}_{x\sim P_t^{\rm sem}}[\nabla_\theta E_\theta(x)]=\mathcal{O}(\xi_{t,k})
\end{equation}
From arguments in Lemma~\ref{lem:alignment} and Eq.~(\ref{eq:app_grad_fpr}), under  Assumption~\ref{ass:nondeg}, we have
\begin{equation}\label{eq:grad-fpr:1}
    \nabla_\theta \text{FPR}_{t,k} = -p(\tau) \cdot \mathbb{E}_{x \sim P^{\text{sem}}_t}[\nabla_\theta E_{\theta_{t,k}}(x)] 
\end{equation}
where $p(\tau) > 0$ is the density of $E_{\theta_t}(x)$ evaluated at $\tau$. By plugging (\ref{eq:app_grad_z:1}) and (\ref{eq:grad-fpr:1}) in (\ref{eq:grad:value:1}), we have
\begin{equation}\label{eq.bound:11}
    \nabla_\theta V(s_{t,k}) = \mathcal{O}(\xi_{t,k}) -\frac{\partial V}{\partial \text{FPR}_{t,k}} \cdot(\rho(\tau)\mathbb{E}_{x \sim P^{\text{sem}}_t}[\nabla_\theta E_{\theta_t}(x)])
\end{equation}
Apply $-c\leq \frac{\partial V}{\partial \text{FPR}_{t,k}}<0$, then
\begin{equation}\label{eq.bound:11:0}
    \nabla_\theta V(s_{t,k})\leq \mathcal{O}(\xi_{t,k}) +c\rho(\tau)\mathbb{E}_{x \sim P^{\text{sem}}_t}[\nabla_\theta E_{\theta_{t,k}}(x)]
\end{equation}
Recall (\ref{F-OW:bound5}) from Subsection~\ref{app:proof_thm:loss-gap bound}: under the Assumptions~\ref{ass:Lipschitz} and \ref{ass:Compactness}, we have 
\begin{align*}
\|\theta_{t,*}^{\text{RL}}-\theta^{\text{GD}}_{t,*}\|^2&\leq\eta_{K,r}+\eta_{\text{RL}}^2 \sum_{k=1}^K \|\nabla_\theta V(s_{t,k-1})\|^2.
\end{align*}
Now apply the upper bound of $ \nabla_\theta V(s_{t,k})$ from (\ref{eq.bound:11:0})
\begin{align}\label{eq.bound:11}
\|\theta_{t,*}^{\text{RL}}-\theta^{\text{GD}}_{t,*}\|^2
&\leq \eta_{K,r}+\eta_{\text{RL}}^2 \sum_{k=1}^K \left[\mathcal{O}(\xi_{t,k-1}) +c\rho(\tau)\mathbb{E}_{x \sim P^{\text{sem}}_t}[\nabla_\theta E_{\theta_{t,k-1}}(x)]\right]\nonumber\\
&= \mathcal{O}({\eta}_{K,r} \xi_t)+\bar{\eta}_{RL}\sum_{k=1}^K\|\mathbb{E}_{x \sim P^{\text{sem}}_t}[\nabla_\theta E_{\theta_{t,k-1}}(x)]\|^2, 
\end{align}
Where $\bar{\eta}_{RL}:=c^2\;\eta^2_{RL} \;\rho^2({\tau})$. Note that without lose of generality, we assume that $\mathcal{O}(\xi_{t,k-1})=\mathcal{O}(\xi_{t})$. Combining the inequalities (\ref{eq.bound:11}) and (\ref{F-OW:bound0}) lead to our main bound  (\ref{eq:corollary:0}). 
\end{proof}

\subsection{Proof of Corollary~\ref{corollary:1}}
\label{app:corollary:1}
\begin{proof} Let us recall (\ref{eq:app_z_tk_def}) from Section ~\ref{app:theory_proofs}:
\begin{equation}
z_{t,k} :=
\mathbb{E}_{x\sim P_{t+1}^{\rm sem}}[E_{\theta_{t,k}}(x)]-
\mathbb{E}_{x\sim P_t^{\rm sem}}[E_{\theta_{t,k}}(x)].
\end{equation}
Hence for consecutive training updates $k$ and $k-1$, we have
\begin{align}
    &z_{t,k}-z_{t,k-1}\nonumber\\
    &= \mathbb E_{x\sim {P}^{sem}_{t+1}}
\big[E_{\theta_{t,k}}(x)\big] - \mathbb E_{x\sim {P}^{sem}_{t}}
\big[E_{\theta_{t,k}}(x)\big] - \mathbb E_{x\sim {P}^{sem}_{t+1}}
\big[E_{\theta_{t,k-1}}(x)\big] 
+\mathbb E_{x\sim {P}^{sem}_{t}}
\big[E_{\theta_{t,k-1}}(x)\big] \nonumber\\
&=\left(\mathbb E_{x\sim {P}^{sem}_{t+1}}
\big[E_{\theta_{t,k}}(x)\big] - \mathbb E_{x\sim {P}^{sem}_{t+1}}
\big[E_{\theta_{t,k-1}}(x)\right)-\left(\mathbb E_{x\sim {P}^{sem}_{t}}
\big[E_{\theta_{t,k}}(x)\big]-\mathbb E_{x\sim {P}^{sem}_{t}}
\big[E_{\theta_{t,k-1}}(x)\big]\right)\nonumber\\
&=\mathbb E_{x\sim {P}^{sem}_{t+1}}\left[E_{\theta_{t,k}}(x)-E_{\theta_{t,k-1}}(x)\right]-\mathbb E_{x\sim {P}^{sem}_{t}}\left[E_{\theta_{t,k}}(x)-E_{\theta_{t,k-1}}(x)\right].\label{eq:corollary:1-1}
\end{align}
Approximate epoch-based energy gradient as $\nabla^{(k)}_\theta E_{\theta_t}(x)\approx E_{\theta_{t,k}}(x)-E_{\theta_{t,k-1}}(x)$, therefore the Eq.~\ref{eq:corollary:1-1} turns into 
\begin{align}
z_{t,k}-z_{t,k-1}=\mathbb E_{x\sim {P}^{sem}_{t+1}}\left[\nabla^{(k)}_\theta E_{\theta_t}(x)\right]-\mathbb E_{x\sim {P}^{sem}_{t}}\left[\nabla^{(k)}_\theta E_{\theta_t}(x)\right]
\end{align}
On the other hand, by Def.~\ref{def:epoch-based energy gradient state}, we have
$$
\|z_{t,k}-z_{t,k-1}\|^2=\| S^{(k)}_{t,\nabla}\|^2, 
$$
where $S^{(k)}_{t,\nabla}$ is the epoch-based energy gradient state. This implies an alternative upper bound for $F^{OW}_t$ as expressed in (\ref{F-OW:bound2}). 
\end{proof}

\subsection{Proof of Theorem~\ref{training-EGF-Gap}}
\label{app:training-EGF-Gap}
\begin{proof} 
We start with the gradient of loss function $\mathcal{L}_G$ under GD and RL-guided optimizers for environment $t$ and training step $k$:
\begin{align}
\mathbf{G}_{t,k}^{GD}:=\nabla_\theta \mathcal{L}_{OW}(\theta^{GD}_{t,k})\;\; \hbox{and}\;\;\;
\mathbf{G}_{t,k}^{RL}:=\nabla_\theta \mathcal{L}_{OW}(\theta^{RL}_{t,k})+\eta_{RL}\nabla_\theta V(s_{t,k}). 
\end{align}
Using the triangle inequality, we bound the length of the gradient difference vector between $\mathbf{G}_{t,k}^{GD}$ and $\mathbf{G}_{t,k}^{RL}$ by
\begin{align}
\|\mathbf{G}_{t,k}^{RL}-\mathbf{G}_{t,k}^{GD}\|_2&=\|\nabla_\theta \mathcal{L}_{OW}(\theta^{RL}_{t,k})+\eta_{RL}\nabla_\theta V(s_{t,k})-\nabla_\theta \mathcal{L}_{OW}(\theta^{GD}_{t,k})\|_2\nonumber\\
&\leq \|\nabla_\theta \mathcal{L}_{OW}(\theta^{RL}_{t,k})-\nabla_\theta \mathcal{L}_{OW}(\theta^{GD}_{t,k})\|_2+\eta_{RL}\|\nabla_\theta V(s_{t,k})\|_2.\label{eq:thm-EGF-train-00}
\end{align}
Next, let us focus on $\|\nabla_\theta V(s_{t,k})\|_2$ and subtract and add the term $\nabla_\theta V(s_{t,k-1})$:
\begin{align*}
\|\nabla_\theta V(s_{t,k})\|_2 &=\|\nabla_\theta V(s_{t,k})-\nabla_\theta V(s_{t,k-1})+\nabla_\theta V(s_{t,k-1})\|_2\nonumber\\
&\leq \|\nabla_\theta V(s_{t,k})-\nabla_\theta V(s_{t,k-1})\|_2+\|\nabla_\theta V(s_{t,k-1})\|_2.
\end{align*}
Recursively, the above arguments lead to the following bound: 
\begin{align}\label{eq:thm-EGF-train:0}
\|\nabla_\theta V(s_{t,k})\|_2\leq \sum_{k'=1}^k \|\nabla_\theta V(s_{t,k'})-\nabla_\theta V(s_{t,k'-1})\|_2 \leq \sqrt{L_V}\sum_{k'=1}^k \|s_{t,k'}-s_{t,k'-1}\|_2, 
\end{align}
where the second inequality is concluded from Assumption~\ref{ass:Lipschitz}. Note that $\|\nabla_\theta V(s_{t,0})\|_2=0$. 

Now, we recall 
\begin{equation}
z_{t,k} :=
\mathbb{E}_{x\sim P_{t+1}^{\rm sem}}[E_{\theta_{t,k}}(x)]-
\mathbb{E}_{x\sim P_t^{\rm sem}}[E_{\theta_{t,k}}(x)].
\end{equation}
By Def.~\ref{def:epoch-based energy gradient state}, and similar arguments in Corollary~\ref{corollary:1}, we redeem the epoch-based gradient energy state
\begin{align}\label{eq:thm-EGF-train:1}
\|z_{t,k'}-z_{t,k'-1}\|_2=\| S^{(k')}_{t,\nabla}\|_2. 
\end{align}
Plugging (\ref{eq:thm-EGF-train:1}) in the upper bound (\ref{eq:thm-EGF-train:0}), we have
\begin{align}\label{V-bound}
\|\nabla_\theta V(s_{t,k})\|_2\leq \sqrt{L_V}\sum_{k'=1}^k \| S^{(k')}_{t,\nabla}\|_2.
\end{align}
Finally, we focus on $\|\nabla_\theta \mathcal{L}_{OW}(\theta^{RL}_{t,k})-\nabla_\theta \mathcal{L}_{OW}(\theta^{GD}_{t,k})\|_2$: Using Assumptions~\ref{ass:Lipschitz} and \ref{ass:Compactness} for $\mathcal{L}_{OW}$, we bound the difference of the gradient of loss for GD and RL-guided optimizers by
\begin{align}\label{L_OW bound}
\|\nabla_\theta \mathcal{L}_{OW}(\theta^{RL}_{t,k})-\nabla_\theta \mathcal{L}_{OW}(\theta^{GD}_{t,k})\|_2\leq 2 \;\sqrt{L} \;r,
\end{align}
where $r$ is the radius of ball $\mathcal{B}(\theta,r)$. We refer Def.~\ref{def:EGF}, apply inverse triangle inequality on $ EGF^{train}_{t,p}$ difference between GD and RL-guided optimizers:
\begin{align}\label{EEGF-difference:1}
\frac{1}{K}\sum_{k=1}^K\|\mathbf{G}_{t,k}^{RL}\|_2-\frac{1}{K}\sum_{k=1}^K\|\mathbf{G}_{t,k}^{GD}\|_2\leq \frac{1}{K}\sum_{k=1}^K \|\mathbf{G}_{t,k}^{RL}-\mathbf{G}_{t,k}^{GD}\|_2. 
\end{align}
Sequentially Plugging \ref{V-bound}) and (\ref{L_OW bound}) in (\ref{eq:thm-EGF-train-00}) and plugging (\ref{eq:thm-EGF-train-00}) in (\ref{EEGF-difference:1}), we conclude
\begin{align}
\frac{1}{K}\sum_{k=1}^K\|\mathbf{G}_{t,k}^{RL}\|_2-\frac{1}{K}\sum_{k=1}^K\|\mathbf{G}_{t,k}^{GD}\|_2&\leq\frac{1}{K} \sum_{k=1}^K \left(2\sqrt{L} r+\eta_{RL} \sqrt{L_V}\sum_{k'=1}^k \| S^{(k')}_{t,\nabla}\|_2\right)\nonumber\\
&\leq 2\sqrt{L} r+\frac{\eta_{RL}\sqrt{L_V}}{K}\sum_{k=1}^K\sum_{k'=1}^k \| S^{(k')}_{t,\nabla}\|_2, 
\end{align}
which proves our claim in (\ref{Eq:training-EGF-Gap}). 
\end{proof}

\subsection{Proof of Theorem~\ref{env-EGF-Gap}}
\label{app:env-EGF-Gap}
\begin{proof}
We start with the gradient of loss function $\mathcal{L}_G$ under GD and RL-guided optimizers for environment $t$:
\begin{align}\label{eq:th.env-EGF:1}
\mathbf{G}_{t}^{GD}:=\nabla_\theta \mathcal{L}_{OW}(\theta^{GD}_{t})\;\; \hbox{and}\;\;\;
\mathbf{G}_{t}^{RL}:=\nabla_\theta \mathcal{L}_{OW}(\theta^{RL}_{t})+\eta_{RL}\nabla_\theta V(s_{t}).
\end{align}
In addition, recall Def.~\ref{def:EGF}, and apply inverse triangle inequality on $ EGF^{env}_{t,p}$ difference between GD and RL-guided optimizers:
\begin{align}\label{env-differece-gradien-flow}
\frac{1}{T}\sum_{t=1}^T\|\mathbf{G}_{t}^{RL}\|_2-\frac{1}{T}\sum_{t=1}^T\|\mathbf{G}_{t}^{GD}\|_2\leq \frac{1}{T}\sum_{t=1}^T \|\mathbf{G}_{t}^{RL}-\mathbf{G}_{t}^{GD}\|_2.
\end{align}
Using $\mathbf{G}_{t}^{GD}$ and $\mathbf{G}_{t}^{RL}$ in (\ref{eq:th.env-EGF:1}), we bound the length of $\mathbf{G}_{t}^{RL}-\mathbf{G}_{t}^{GD}$ by triangle inequality as follows:
\begin{align}\label{eq:thm-env-EGF:00}
\|\mathbf{G}_{t}^{RL}-\mathbf{G}_{t}^{GD}\|_2&=\|\nabla_\theta \mathcal{L}_{OW}(\theta^{RL}_{t})+\eta_{RL}\nabla_\theta V(s_{t})-\nabla_\theta \mathcal{L}_{OW}(\theta^{GD}_{t})\|_2\nonumber\\
&\leq \|\nabla_\theta \mathcal{L}_{OW}(\theta^{RL}_{t})-\nabla_\theta \mathcal{L}_{OW}(\theta^{GD}_{t})\|_2+\eta_{RL}\|\nabla_\theta V(s_{t})\|_2.
\end{align}
Using Assumptions~\ref{ass:Lipschitz} and \ref{ass:Compactness} for $\mathcal{L}_{OW}$ in environment $t$, we have 
\begin{align}\label{L_OW bound:2}
\|\nabla_\theta \mathcal{L}_{OW}(\theta^{RL}_{t})-\nabla_\theta \mathcal{L}_{OW}(\theta^{GD}_{t})\|_2\leq 2 \;\sqrt{L} \;r,
\end{align}
where $r$ in the radius of $\mathcal{B}(\theta_t,r)$ in Assumption~\ref{ass:Compactness}. Next, similar to arguments in Theorem~\ref{training-EGF-Gap}, we have 
\begin{align}\label{eq:thm-env-EGF:3}
&\|\nabla_\theta V(s_{t})\|_2 =\|\nabla_\theta V(s_{t})-\nabla_\theta V(s_{t-1})+\nabla_\theta V(s_{t-1})\|_2\nonumber\\
&\leq \|\nabla_\theta V(s_{t})-\nabla_\theta V(s_{t-1})\|_2+\|\nabla_\theta V(s_{t-1})\|_2
=\sum_{\tau=1}^t \|\nabla_\theta V(s_{\tau})-\nabla_\theta V(s_{\tau-1})\|_2\nonumber\\
&\leq \sum_{\tau=1}^t \sqrt{L_V}\|s_{\tau}-s_{\tau-1}\|_2\;\;\hbox{using  Assumption~\ref{ass:Compactness}}.
\end{align}
Note that $\|\nabla_\theta V(s_{0})\|_2=0$. Recall
\begin{equation}
z_{\tau} :=
\mathbb{E}_{x\sim P_{\tau+1}^{\rm sem}}[E_{\theta_{\tau}}(x)]-
\mathbb{E}_{x\sim P_\tau^{\rm sem}}[E_{\theta_{\tau}}(x)].
\end{equation}
Hence for two environments $\tau-1$ and $\tau$:
\begin{align}
&z_{\tau}-z_{\tau-1} \nonumber\\
    &= \mathbb E_{x\sim {P}^{sem}_{\tau+1}}
\big[E_{\theta_{\tau}}(x)\big] - \mathbb E_{x\sim {P}^{sem}_{\tau}}
\big[E_{\theta_{\tau}}(x)\big] - \mathbb E_{x\sim {P}^{sem}_{\tau}}
\big[E_{\theta_{\tau-1}}(x)\big] 
+\mathbb E_{x\sim {P}^{sem}_{\tau-1}}
\big[E_{\theta_{\tau-1}}(x)\big] \nonumber\\
&=\left(\mathbb E_{x\sim {P}^{sem}_{\tau+1}}
\big[E_{\theta_{\tau}}(x)\big]-\mathbb E_{x\sim {P}^{sem}_{\tau}}
\big[E_{\theta_{\tau-1}}(x)\big]\right)-\left(\mathbb E_{x\sim {P}^{sem}_{\tau}}\big[E_{\theta_{\tau}}(x)\big]-\mathbb E_{x\sim {P}^{sem}_{\tau-1}}\big[E_{\theta_{\tau-1}}(x)\big] \right)\nonumber\\
&=S_{\tau\rightarrow \tau+1}-E_{\tau-1\rightarrow \tau},\label{eq:thm-env-EGF:2}
\end{align}
The last equality of (\ref{eq:thm-env-EGF:2}) is based on Def.~\ref{def:env-based energy state}. Applying Assumptions~\ref{ass:Slow environment-transfer generalization energy} and \ref{ass:Slow environment-based semantic shifted energy}
\begin{align}\label{eq:thm-env-EGF:4}
\|z_{\tau}-z_{\tau-1}\|_2\leq \mathcal{O}(\xi^S_{\tau\rightarrow \tau+1})+\mathcal{O}(\xi^E_{\tau-1\rightarrow \tau})
\end{align}
Going back to (\ref{eq:thm-env-EGF:3}), substituting (\ref{eq:thm-env-EGF:4}) in (\ref{eq:thm-env-EGF:3}) implies 
\begin{align}\label{grad-value:eq2}
\|\nabla_\theta V(s_{t})\|_2 \leq \sqrt{L_V}\sum_{\tau=1}^t \left(\mathcal{O}(\xi^s_{\tau\rightarrow \tau+1})+\mathcal{O}(\xi^E_{\tau-1\rightarrow \tau})\right).
\end{align}
Combining (\ref{L_OW bound:2}) and (\ref{grad-value:eq2}) and use them in (\ref{eq:thm-env-EGF:00}), we have
\begin{align}\label{eq:thm-env-EGF:4}
\|\mathbf{G}_{t}^{RL}-\mathbf{G}_{t}^{GD}\|_2\leq 2 \;\sqrt{L} \;r+\eta_{RL} \sqrt{L_V}\sum_{\tau=1}^t \left(\mathcal{O}(\xi^s_{\tau\rightarrow \tau+1})+\mathcal{O}(\xi^E_{\tau-1\rightarrow \tau})\right).
\end{align}
Hence by using (\ref{eq:thm-env-EGF:4}) in (\ref{env-differece-gradien-flow}), we derive the upper bound below:
\begin{align}
&\frac{1}{T}\sum_{t=1}^T\|\mathbf{G}_{t}^{RL}\|_2-\frac{1}{T}\sum_{t=1}^T\|\mathbf{G}_{t}^{GD}\|_2\leq
2 \;\sqrt{L} \;r+ \frac{\eta_{RL} \sqrt{L_V}}{T}\sum_{t=1}^T \sum_{\tau=1}^t \left(\mathcal{O}(\xi^s_{\tau\rightarrow \tau+1})+\mathcal{O}(\xi^E_{\tau-1\rightarrow \tau})\right),
\end{align}
which proves the theorem and yield our claim in (\ref{eq:env-EGF-Gap}).
\end{proof}